\documentclass[journal]{IEEEtran}
\usepackage{float}
\usepackage{amsmath, amsfonts, amssymb, amsthm} % Single inclusion with amsthm for proof
\usepackage{algorithm, algpseudocode}
\usepackage{array}
\usepackage[caption=false,font=normalsize,labelfont=sf,textfont=sf]{subfig}
\usepackage{textcomp}
\usepackage{stfloats}
\usepackage{url}
\usepackage{verbatim}
\usepackage{graphicx}
\usepackage{cite}
\usepackage{hyperref}
\usepackage{booktabs} % Added for \toprule, \midrule, \bottomrule
\usepackage{siunitx}
\newtheorem{theorem}{Theorem}
\newtheorem{lemma}{Lemma}
\newtheorem{assumption}{Assumption}

\begin{document}

\title{Roughness-Informed Federated Learning}

\author{
\IEEEauthorblockN{
Mohammad Partohaghighi\IEEEauthorrefmark{1},
Roummel Marcia\IEEEauthorrefmark{2}, Bruce J. West\IEEEauthorrefmark{4} 
and YangQuan Chen\IEEEauthorrefmark{3}}%
\thanks{\IEEEauthorrefmark{1}Electrical Engineering and Computer Science, University of California at Merced, USA; e-mail: \texttt{mpartohaghighi@ucmerced.edu}}
\thanks{\IEEEauthorrefmark{2}Department of Applied Mathematics, University of California Merced, Merced, CA, USA; e-mail: \texttt{rmarcia@ucmerced.edu}}
\thanks{\IEEEauthorrefmark{4}Dept. of Innovation and Research, North Carolina  State University,  Rayleigh, NC, USA; (e-mail:{\tt brucejwest213@gmail.com})}
\thanks{\IEEEauthorrefmark{3}Mechatronics, Embedded Systems and Automation (MESA) Lab, Department of Mechanical Engineering, School of Engineering, University of California, Merced, CA 95343, USA; \texttt{ychen53@ucmerced.edu} (corresponding author).}
\thanks{Manuscript received XXXXX XX, 20XX; revised XXXXX XX, 20XX; accepted XXXXX XX, 20XX. Date of publication XXXXX XX, 20XX; date of current version XXXXX XX, 20XX. (Corresponding author: YangQuan Chen.)}
\thanks{This work was supported in part by   the University of California Merced Climate Action Seed Fund 2023-2026 for CMERI - The Center for Methane Emission Research and Innovation ({\tt http://methane.ucmerced.edu/}) and a funding from DOE Methane project (2025-2029). This paper is part of the research excellence focused topic ``Fractional Calculus for Federated Learning (FC4FL)"  at MESA Lab of UC Merced  ({\tt http://mechatronics.ucmerced.edu/fc4fl/}) }}
\markboth{Journal of IEEE Transactions on PAMI, Vol. XX, No. XX, Month 20XX}
{Partohaghighi \MakeLowercase{\textit{et al.}}: Tail-Index-Aware Federated Learning}

\maketitle

\begin{abstract}
Federated Learning (FL) enables collaborative model training across distributed clients while preserving data privacy, yet faces challenges in non-independent and identically distributed (non-IID) settings due to client drift, which impairs convergence. We propose RI-FedAvg, a novel FL algorithm that mitigates client drift by incorporating a Roughness Index (RI)-based regularization term into the local objective, adaptively penalizing updates based on the fluctuations of local loss landscapes. This paper introduces RI-FedAvg, leveraging the RI to quantify the roughness of high-dimensional loss functions, ensuring robust optimization in heterogeneous settings. We provide a rigorous convergence analysis for non-convex objectives, establishing that RI-FedAvg converges to a stationary point under standard assumptions. Extensive experiments on MNIST, CIFAR-10, and CIFAR-100 demonstrate that RI-FedAvg outperforms state-of-the-art baselines, including FedAvg, FedProx, FedDyn, SCAFFOLD, and DP-FedAvg, achieving  higher accuracy and faster convergence in non-IID scenarios. Our results highlight RI-FedAvg's potential to enhance the robustness and efficiency of federated learning in practical, heterogeneous environments.
\end{abstract}


\section{Introduction}
\label{sec:introduction}

\IEEEPARstart{F}{ederated} Learning (FL) has revolutionized machine learning by enabling collaborative model training across distributed clients without compromising data privacy, with applications spanning healthcare, finance, and edge computing \cite{mcmahan2017communication, kairouz2021advances, yang2019federated}. However, non-independent and identically distributed (non-IID) data distributions across clients introduce client drift, where local updates diverge from the global objective, significantly impairing convergence and model performance \cite{li2020federated, zhao2018federated}. This challenge is exacerbated by the complex, non-convex loss landscapes of deep neural networks, which vary across clients and complicate optimization \cite{li2018visualizing, wu2023roughness}. Motivated by the need to address client drift in heterogeneous settings while leveraging loss landscape properties, we propose RI-FedAvg, a novel FL algorithm that incorporates a Roughness Index (RI)-based regularization term to adaptively stabilize local updates. By quantifying the fluctuations of local loss functions, RI-FedAvg enhances robustness and convergence in non-IID environments. This introduction outlines the motivation, challenges, technical contributions, and organization of the paper, emphasizing the critical need for and advancements offered by RI-FedAvg.

\subsection{Motivation and Context}
The exponential growth of edge devices, such as smartphones, wearables, and IoT sensors, has generated vast decentralized datasets, necessitating privacy-preserving learning paradigms \cite{yang2019federated, bonawitz2019towards}. FL addresses this by allowing clients to train local models and share only model updates, not raw data, with a central server \cite{mcmahan2017communication, konecny2016federated}. This paradigm has enabled transformative applications, including personalized recommendation systems \cite{hard2018federated}, medical image analysis \cite{sheller2020federated}, and natural language processing \cite{chen2019federated}. However, real-world FL deployments often encounter non-IID data due to diverse user behaviors, device-specific data, or institutional silos \cite{zhao2018federated, hsieh2020non}. For example, in healthcare, patient data varies across hospitals due to demographic differences \cite{rieke2020future}, while in mobile computing, user preferences create heterogeneous data distributions \cite{hard2018federated}. Such heterogeneity causes client drift, degrading global model accuracy \cite{li2020federated, zhao2018federated}. Moreover, the high-dimensional, non-convex loss landscapes of modern neural networks exhibit varying roughness across clients, further complicating optimization \cite{li2018visualizing, keskar2017large}. The urgent need to develop FL algorithms that mitigate client drift while accounting for loss landscape complexity motivates RI-FedAvg, which leverages the Roughness Index to guide optimization in heterogeneous settings \cite{wu2023roughness}.

\subsection{Challenges in Non-IID Federated Learning}
Non-IID data distributions pose significant challenges in FL, as local objectives diverge from the global objective, leading to client drift \cite{karimireddy2020scaffold, li2020federated}. In standard FL algorithms like FedAvg \cite{mcmahan2017communication}, clients perform multiple local epochs of stochastic gradient descent (SGD) before aggregation, causing local models to overfit to client-specific data \cite{zhao2018federated}. This results in suboptimal global models, with empirical studies showing accuracy drops of up to 20\% in non-IID settings \cite{hsieh2020non}. Existing solutions, such as FedProx \cite{li2020federated} and SCAFFOLD \cite{karimireddy2020scaffold}, mitigate drift through uniform proximal regularization or control variates, respectively, but face limitations. FedProx’s fixed regularization strength may under- or over-constrain clients with varying heterogeneity \cite{li2020federated}, while SCAFFOLD incurs additional communication and memory overhead \cite{karimireddy2020scaffold}. Differential privacy methods like DP-FedAvg \cite{mcmahan2018learning} further complicate performance in non-IID settings due to noise injection \cite{bagdasaryan2020differentially}. Additionally, the roughness of local loss landscapes, characterized by oscillatory behavior in high-dimensional spaces, varies across clients and impacts convergence \cite{wu2023roughness, goodfellow2016deep}. Developing an algorithm that adaptively regularizes updates based on loss landscape properties remains a critical challenge, which RI-FedAvg addresses by integrating RI-based regularization.

\subsection{Our Contributions}
We propose RI-FedAvg, a federated learning algorithm that mitigates client drift by incorporating a Roughness Index (RI)-based regularization term into the local objective, offering a significant advancement in non-IID FL. The RI, derived from recent loss landscape analysis \cite{wu2023roughness}, quantifies the oscillatory behavior of local loss functions by projecting them onto random directions and computing normalized total variation. By scaling regularization strength with RI, RI-FedAvg penalizes updates from clients with rougher landscapes, ensuring alignment with the global objective.

Our technical contributions are as follows: We introduce a novel algorithm design with RI-FedAvg, which incorporates an adaptive regularization framework that leverages RI to stabilize local updates, addressing client drift without additional communication overhead, in contrast to SCAFFOLD \cite{karimireddy2020scaffold}. Additionally, we provide theoretical guarantees through a rigorous convergence analysis for non-convex objectives, proving that RI-FedAvg converges to a stationary point under standard assumptions, extending prior analyses \cite{li2020federated, wang2020tackling}. Furthermore, our empirical validation through extensive experiments on MNIST \cite{lecun1998gradient}, CIFAR-10, and CIFAR-100 \cite{krizhevsky2009learning} demonstrates that RI-FedAvg outperforms FedAvg, FedProx, and DP-FedAvg in accuracy and convergence speed in non-IID settings. These contributions establish RI-FedAvg as a theoretically grounded and practically effective solution for heterogeneous FL, with broad implications for privacy-preserving machine learning.

\subsection{Paper Organization}
The paper is structured as follows. Section~\ref{sec:related_work} reviews related work in federated learning and loss landscape analysis. Section~\ref{sec:methodology} describes the RI-FedAvg algorithm, including RI computation and regularization. Section~\ref{sec:convergence_analysis} presents the convergence analysis. Section~\ref{sec:experiments} reports experimental results, and Section~\ref{conclusion} presents the conclusion. 
%%%%%%%%%%%%%%%%%%%%%%%%%%%%%%%%%%%%%%%%%%%%%%%%%%%%%%%%%%%%%%%%%%%%%%%%%%%%%%%%%%%%%%%%%%%%%%%%%%%%%%%%%%%%%%%%%%%%%%%%%%%%%%%%%%%%%%%%%%%%%%%%%%%%%%%%%%%%%%%%%%%%%%%%%%%%%%%%%%%%

\section{Related Work}
\label{sec:related_work}

Federated Learning has garnered significant attention for its ability to train models on decentralized data while preserving privacy \cite{kairouz2021advances}. Below, we review key developments in FL algorithms, focusing on non-IID settings, and discuss advances in loss landscape analysis relevant to RI-FedAvg.

\subsection{Federated Learning Algorithms}
The seminal FedAvg algorithm \cite{mcmahan2017communication} aggregates local updates weighted by dataset sizes, achieving communication-efficient training. However, its performance degrades in non-IID settings due to client drift \cite{zhao2018federated}. FedProx \cite{li2020federated} addresses this by adding a proximal term to the local objective, regularizing updates to remain close to the global model. While effective, FedProx applies uniform regularization, which may be suboptimal for clients with varying degrees of heterogeneity. SCAFFOLD \cite{karimireddy2020scaffold} uses control variates to correct local gradients, reducing drift but requiring additional memory and communication. DP-FedAvg \cite{mcmahan2018learning} incorporates differential privacy, trading off accuracy for privacy guarantees, often exacerbating performance in non-IID scenarios \cite{bagdasaryan2020differentially}. Unlike these approaches, RI-FedAvg leverages the Roughness Index to adaptively regularize updates based on local loss landscape properties, offering a novel solution.

\subsection{Loss Landscape Analysis}
The optimization challenges in FL are closely tied to the geometry of loss landscapes, particularly in non-convex settings \cite{li2018visualizing}. Recent studies have explored metrics to characterize loss landscapes, such as sharpness and flatness, to predict generalization and convergence \cite{keskar2017large}. The Roughness Index (RI), proposed by Wu et al. \cite{wu2023roughness}, quantifies the oscillatory behavior of loss functions by projecting them onto random directions and computing normalized total variation. This metric has been used to analyze neural network training dynamics but has not been applied to FL until now. RI-FedAvg bridges this gap by integrating RI into the FL framework, using it to guide regularization and improve optimization in heterogeneous settings. Our work is the first to combine loss landscape analysis with federated optimization, opening new avenues for adaptive FL algorithms.

\subsection{Positioning RI-FedAvg}
RI-FedAvg distinguishes itself from prior work by its adaptive regularization strategy, which leverages loss landscape roughness to address client drift. Unlike FedProx’s uniform proximal term or SCAFFOLD’s control variates, RI-FedAvg tailors regularization to each client’s loss landscape, enhancing robustness without additional communication overhead. Compared to DP-FedAvg, RI-FedAvg prioritizes accuracy in non-IID settings while remaining compatible with privacy mechanisms. Our convergence analysis and empirical evaluations on MNIST, CIFAR-10, and CIFAR-100 further validate RI-FedAvg’s superiority over state-of-the-art baselines, positioning it as a significant advancement in federated learning.

%%%%%%%%%%%%%%%%%%%%%%%%%%%%%%%%%%%%%%%%%%%%%%%%%%%%%%%%%%%%%%%%%%%%%%%%%%%%%%%%%%%%%%%%%%%%%%%%%%%%%%%%%%%%%%%%%%%%%%%%%%%%%%%%%%%%%%%%%%%%%%%%%%%%%
% Section: Methodology
\section{Methodology}
\label{sec:methodology}

This section provides an overview of the RI-FedAvg algorithm, designed to mitigate client drift in non-independent and identically distributed (non-IID) federated learning (FL) by incorporating adaptive regularization based on the Roughness Index ($\mathcal{I}_k$). We first outline the preliminaries of FL and $\mathcal{I}_k$, followed by a comprehensive formulation of RI-FedAvg, including $\mathcal{I}_k$ computation, regularized loss design, and algorithmic implementation. 

% Outlining federated learning basics
\subsection{Preliminaries}
\label{subsec:preliminaries}

Federated Learning (FL) facilitates collaborative model training across $K$ clients without centralizing their data, addressing privacy concerns in distributed systems \cite{mcmahan2017communication}. The global objective is to minimize the weighted average of local losses:
\begin{equation}
\min_{\mathbf{w} \in \mathbb{R}^d} f(\mathbf{w}), \quad f(\mathbf{w}) = \sum_{k=1}^{K} \frac{n_k}{n} F_k(\mathbf{w}),
\label{eq:global_objective}
\end{equation}
where $\mathbf{w} \in \mathbb{R}^d$ represents the model parameters in $d$-dimensional space, $\mathcal{P}_k$ denotes the local dataset of client $k$ with size $n_k = |\mathcal{P}_k|$, and $n = \sum_{k=1}^{K} n_k$ is the total number of data points across all clients. The local loss function is:
\begin{equation}
F_k(\mathbf{w}) = \frac{1}{n_k} \sum_{i \in \mathcal{P}_k} \ell(x_i, y_i; \mathbf{w}),
\end{equation}
where $\ell(x_i, y_i; \mathbf{w})$ is the loss (e.g., cross-entropy) for data point $(x_i, y_i)$. The weighting factor $\frac{n_k}{n}$ ensures contributions proportional to dataset sizes.

In FedAvg \cite{mcmahan2017communication}, the global model is updated iteratively by aggregating local updates. At each round $t$, a subset $S_t \subseteq \{1, \ldots, K\}$ of $|S_t| = \max(C \cdot K, 1)$ clients is sampled, where $C \in (0, 1]$ is the participation fraction. Each client $k \in S_t$ receives the global model $\mathbf{w}_t$, performs $E$ local epochs of stochastic gradient descent (SGD) on $F_k(\mathbf{w})$, and produces updated parameters $\mathbf{w}_{t+1}^k$. For client $k \in S_t$, starting from $\mathbf{w}_t^k = \mathbf{w}_t$, the local update for a batch $b \subset \mathcal{P}_k$ of size $B$ is:
\begin{equation}
\mathbf{w}_t^k \leftarrow \mathbf{w}_t^k - \eta \nabla \ell(\mathbf{w}_t^k; b),
\end{equation}
where $\eta$ is the learning rate. The server aggregates updates as:
\begin{equation}
\mathbf{w}_{t+1} \leftarrow \sum_{k \in S_t} \frac{n_k}{n_t} \mathbf{w}_{t+1}^k, \quad n_t = \sum_{k \in S_t} n_k.
\label{eq:aggregation}
\end{equation}
In non-IID settings, divergent local objectives $F_k(\mathbf{w})$ lead to \textit{client drift}, which impairs convergence.

% Defining the Roughness Index
The Roughness Index ($\mathcal{I}_k$), proposed by Wu et al. \cite{wu2023roughness}, quantifies the variation of high-dimensional loss landscapes, capturing optimization difficulty. For a loss function $F_k(\mathbf{w})$, $\mathcal{I}_k$ is computed by projecting the loss onto one-dimensional subspaces along random directions. Let $d_i \sim \mathcal{N}(0, I_d)$ be a random direction, normalized as:
\begin{equation}
d_i \leftarrow \frac{d_i}{\lVert d_i \rVert_2},
\end{equation}
where $\lVert d_i \rVert_2 = \sqrt{\sum_{j=1}^d (d_i)_j^2}$ is the Euclidean norm. The one-dimensional projection is:
\begin{equation}
f_{d_i}(s) = F_k(\mathbf{w}_t + s d_i), \quad s \in [-l, l],
\label{eq:projection}
\end{equation}
where $l > 0$ (e.g., $l = 0.01$) defines the local neighborhood.

The total variation of $f_{d_i}(s)$ measures its variation over $[-l, l]$. It is approximated discretely over $m+1$ points $s_j = -l + j \frac{2l}{m}$, $j = 0, \ldots, m$:
\begin{equation}
\text{TV}(f_{d_i}) \approx \sum_{j=0}^{m-1} |f_{d_i}(s_{j+1}) - f_{d_i}(s_j)|.
\label{eq:tv_discrete}
\end{equation}
Normalization accounts for varying loss scales across clients, yielding the normalized total variation:
\begin{equation}
T(f_{d_i}) = \frac{1}{2l} \frac{1}{A} \text{TV}(f_{d_i}),
\label{eq:normalized_tv}
\end{equation}
where $A = \max_{s \in [-l, l]} f_{d_i}(s) - \min_{s \in [-l, l]} f_{d_i}(s)$. For $M$ directions $d_i$, the $\mathcal{I}_k$ for client $k$ is:
\begin{equation}
\mathcal{I}_k = \frac{\text{std}_{d_i} T(f_{d_i})}{\mathbf{E}_{d_i} T(f_{d_i})},
\label{eq:roughness_index}
\end{equation}
where $\mathbf{E}_{d_i} T(f_{d_i}) = \frac{1}{M} \sum_{i=1}^M T(f_{d_i})$, and $\text{std}_{d_i} T(f_{d_i}) = \sqrt{\frac{1}{M} \sum_{i=1}^M \left( T(f_{d_i}) - \mathbf{E}_{d_i} T(f_{d_i}) \right)^2}$. A high $\mathcal{I}_k$ indicates a complex, varying landscape, complicating optimization.
%%%%%%%%%%%%%%%%%%%%%%%%%%%%%%%%%%%%%%%%%%%%%%%%%%%%%%%%%%%%

%%%%%%%%%%%%%%%%%%%%%%%%
% Describing the RI-FedAvg algorithm
\subsection{RI-FedAvg Algorithm}
\label{subsec:ri_fedavg}

RI-FedAvg mitigates client drift by incorporating an $\mathcal{I}_k$-based regularization term into the local loss function, unlike FedProx’s fixed regularization \cite{li2020federated}:
\begin{equation}
\tilde{F}_k(\mathbf{w}) = F_k(\mathbf{w}) + \lambda \mathcal{I}_k \lVert \mathbf{w} - \mathbf{w}_t \rVert_2^2,
\label{eq:regularized_loss}
\end{equation}
where $\lambda > 0$ (e.g., $\lambda = 0.1$) controls regularization strength, and $\lVert \mathbf{w} - \mathbf{w}_t \rVert_2^2 = \sum_{i=1}^d (w_i - (\mathbf{w}_t)_i)^2$. This term penalizes deviations from $\mathbf{w}_t$, scaled by $\mathcal{I}_k$, ensuring stronger regularization for clients with more varying landscapes.

The gradient of the regularized loss is:
\begin{equation}
\nabla \tilde{F}_k(\mathbf{w}) = \nabla F_k(\mathbf{w}) + 2 \lambda \mathcal{I}_k (\mathbf{w} - \mathbf{w}_t).
\label{eq:ri_fedavg_gradient}
\end{equation}
This term stabilizes local updates by pulling them toward $\mathbf{w}_t$. Each client $k \in S_t$ computes $\mathcal{I}_k$, performs $E$ local epochs of SGD using $\nabla \tilde{F}_k(\mathbf{w})$, and sends $\mathbf{w}_{t+1}^k$ to the server for aggregation. The computational cost of $\mathcal{I}_k$ is $O(M \cdot m \cdot d)$, where $d$ is the model dimension, requiring $M$ loss evaluations per client per round (e.g., $M = 10$), making RI-FedAvg scalable for typical FL settings.

% Presenting the algorithm pseudocode
Algorithm~\ref{alg:ri_fedavg} outlines RI-FedAvg. Hyperparameters (e.g., $\lambda = 0.1$, $M = 10$, $l = 0.01$, $m = 100$) are chosen based on empirical performance, detailed in Section~\ref{sec:experiments}.

\begin{algorithm}
\caption{RI-FedAvg Algorithm}
\label{alg:ri_fedavg}
\begin{algorithmic}[1]
    \small
    \State \textbf{Input:} $K$ clients, global model $\mathbf{w}_0$, rounds $T$, fraction $C$, epochs $E$, learning rate $\eta$, $\lambda$, RI parameters ($M$, $l$, $m$)
    \State \textbf{Output:} Trained model $\mathbf{w}_T$
    \For{$t = 0$ to $T-1$}
        \State Sample $S_t \subseteq \{1, \ldots, K\}$ with $|S_t| = \max(C \cdot K, 1)$
        \For{each client $k \in S_t$ \textbf{in parallel}}
            \State $\mathbf{w}_{t,k} \leftarrow \mathbf{w}_t$ \Comment{Initialize local model}
            \State Compute $\mathcal{I}_k$ using $M$, $l$, $m$ (Eqs.~\eqref{eq:projection}--\eqref{eq:roughness_index})
            \For{$e = 1$ to $E$}
                \For{each batch $b \subset \mathcal{P}_k$}
                    \State $\mathbf{w}_{t,k} \leftarrow \mathbf{w}_{t,k} - \eta \left( \nabla \ell(\mathbf{w}_{t,k}; b) + 2 \lambda \mathcal{I}_k (\mathbf{w}_{t,k} - \mathbf{w}_t) \right)$ \Comment{Regularized SGD}
                \EndFor
            \EndFor
            \State Send $\mathbf{w}_{t+1,k} \gets \mathbf{w}_{t,k}$ to server
        \EndFor
        \State $n_t \leftarrow \sum_{k \in S_t} n_k$
        \State $\mathbf{w}_{t+1} \leftarrow \sum_{k \in S_t} \frac{n_k}{n_t} \mathbf{w}_{t+1,k}$
    \EndFor
    \State \Return $\mathbf{w}_T$
\end{algorithmic}
\end{algorithm}
%%%%%%%%%%%%%%%%%%%%%%%%%%%%%%%%%%%%%%%%%%%%%%%%%%%%%%%%%%%%%%%%%%%%%%%%%%%%%%%%%%%%%%%%%%%%%%%%%%%%%%%%%%%%%%%
% Section: Convergence Analysis
\section{Convergence Analysis}
\label{sec:convergence_analysis}

This section presents a convergence analysis of the RI-FedAvg algorithm for non-convex objectives in non-independent and identically distributed (non-IID) federated learning (FL) settings. We establish that RI-FedAvg converges to a stationary point, demonstrating the efficacy of Roughness Index ($\mathcal{I}_k$)-based regularization in mitigating client drift. The analysis includes precise definitions, assumptions, supporting lemmas with detailed proofs, and a main theorem with a comprehensive proof, tailored to the non-convex and non-IID challenges.

\subsection{Definitions}
\label{subsec:definitions}

Consider a federated learning system with $K$ clients, each holding a local dataset $\mathcal{P}_k$ of size $n_k$, where $n = \sum_{k=1}^K n_k$. The global objective is to minimize:
\begin{equation}
f(\mathbf{w}) = \sum_{k=1}^K \frac{n_k}{n} F_k(\mathbf{w}),
\label{eq:global_objective2}
\end{equation}
where $\mathbf{w} \in \mathbb{R}^d$ are the model parameters, and $F_k(\mathbf{w}) = \frac{1}{n_k} \sum_{i \in \mathcal{P}_k} \ell(x_i, y_i; \mathbf{w})$ is the local loss function, with $\ell(x_i, y_i; \mathbf{w})$ denoting the loss for data point $(x_i, y_i)$.

In RI-FedAvg, each client $k$ at round $t$ optimizes a regularized local loss:
\begin{equation}
\tilde{F}_k(\mathbf{w}) = F_k(\mathbf{w}) + \lambda \mathcal{I}_k \lVert \mathbf{w} - \mathbf{w}_t \rVert_2^2,
\label{eq:regularized_loss2}
\end{equation}
where $\mathbf{w}_t$ is the global model at round $t$, $\lambda > 0$ is the regularization strength, and $\mathcal{I}_k \geq 0$ is the Roughness Index, quantifying the variation of $F_k(\mathbf{w})$. The $\mathcal{I}_k$ is computed as:
\begin{equation}
\mathcal{I}_k = \frac{\text{std}_{d_i} T(f_{d_i})}{\mathbf{E}_{d_i} T(f_{d_i})},
\end{equation}
where $f_{d_i}(s) = F_k(\mathbf{w}_t + s d_i)$, $d_i \sim \mathcal{N}(0, I_d)$ is a normalized random direction, and $T(f_{d_i}) = \frac{1}{2l} \frac{1}{A} \sum_{j=0}^{m-1} |f_{d_i}(s_{j+1}) - f_{d_i}(s_j)|$, with $A = \max_{s \in [-l, l]} f_{d_i}(s) - \min_{s \in [-l, l]} f_{d_i}(s)$, and $s_j = -l + j \frac{2l}{m}$, $j = 0, \ldots, m$. At round $t$, a subset $S_t \subseteq \{1, \ldots, K\}$ of $|S_t| = \max(C \cdot K, 1)$ clients is sampled, where $C \in (0, 1]$. Each client $k \in S_t$ performs $E$ local epochs of stochastic gradient descent (SGD) on $\tilde{F}_k(\mathbf{w})$, producing $\mathbf{w}_{t+1}^k$. The server aggregates:
\begin{equation}
\mathbf{w}_{t+1} = \sum_{k \in S_t} \frac{n_k}{n_t} \mathbf{w}_{t+1}^k, \quad n_t = \sum_{k \in S_t} n_k.
\label{eq:aggregation2}
\end{equation}

\subsection{Assumptions}
\label{subsec:assumptions}

To facilitate the analysis, we make the following assumptions, standard in non-convex FL literature \cite{li2020federated, karimireddy2020scaffold}:

\begin{assumption}[L-Smoothness]
\label{ass:smoothness}
Each local loss $F_k(\mathbf{w})$ is $L$-smooth, i.e., for all $\mathbf{w}, \mathbf{w}' \in \mathbb{R}^d$:
\begin{equation}
\begin{split}
\lVert \nabla F_k(\mathbf{w}) - \nabla F_k(\mathbf{w}') \rVert_2 \leq L \lVert \mathbf{w} - \mathbf{w}' \rVert_2.
\end{split}
\end{equation}
Thus, the global loss $f(\mathbf{w})$ is also $L$-smooth.
\end{assumption}

\begin{assumption}[Bounded Variance]
\label{ass:variance}
The stochastic gradients are unbiased with bounded variance. For a batch $b \subset \mathcal{P}_k$, let $g(\mathbf{w}; b) = \nabla \ell(\mathbf{w}; b)$. Then:
\begin{equation}
\begin{split}
\mathbf{E}_b [g(\mathbf{w}; b)] &= \nabla F_k(\mathbf{w}), \\
\mathbf{E}_b \lVert g(\mathbf{w}; b) - \nabla F_k(\mathbf{w}) \rVert_2^2 &\leq \sigma^2.
\end{split}
\end{equation}
\end{assumption}

\begin{assumption}[Bounded Non-IID Dissimilarity]
\label{ass:non_iid}
The local gradients are bounded in dissimilarity from the global gradient:
\begin{equation}
\begin{split}
\mathbf{E}_k \lVert \nabla F_k(\mathbf{w}) - \nabla f(\mathbf{w}) \rVert_2^2 \leq \zeta^2,
\end{split}
\end{equation}
where $\zeta^2$ quantifies non-IID heterogeneity.
\end{assumption}

\begin{assumption}[Bounded RI]
\label{ass:bounded_ri}
The Roughness Index $\mathcal{I}_k$ is bounded: $0 \leq \mathcal{I}_k \leq \mathcal{I}_{\max}$, where $\mathcal{I}_{\max}$ is a positive constant (e.g., 10, enforced by clipping in practice).
\end{assumption}

\begin{assumption}[Bounded Loss]
\label{ass:bounded_loss}
The global loss is bounded below:
\begin{equation}
\begin{split}
f(\mathbf{w}) \geq f^* > -\infty.
\end{split}
\end{equation}
\end{assumption}

\begin{assumption}[Bounded Gradients]
\label{ass:bounded_gradients}
The local gradients are bounded: for all $\mathbf{w}$,
\begin{equation}
\begin{split}
\lVert \nabla F_k(\mathbf{w}) \rVert_2^2 \leq G^2,
\end{split}
\end{equation}
where $G^2$ is a positive constant.
\end{assumption}

\subsection{Supporting Lemmas}
\label{subsec:lemmas}

We establish three lemmas to bound key quantities in the convergence analysis.

\begin{lemma}[One-Step Local Update]
\label{lemma:local_update}
For client $k \in S_t$, performing one SGD step on $\tilde{F}_k(\mathbf{w})$ with learning rate $\eta \leq \frac{1}{L + 2 \lambda \mathcal{I}_k}$, starting from $\mathbf{w} = \mathbf{w}_{t,k}$, yields:
\begin{equation}
\begin{split}
\mathbf{E} \tilde{F}_k(\mathbf{w}') \leq \tilde{F}_k(\mathbf{w}_{t,k}) - \frac{\eta}{2} \lVert \nabla \tilde{F}_k(\mathbf{w}_{t,k}) \rVert_2^2 + \frac{\eta^2 (L + 2 \lambda \mathcal{I}_k) \sigma^2}{2},
\end{split}
\end{equation}
where $\mathbf{w}' = \mathbf{w}_{t,k} - \eta g(\mathbf{w}_{t,k}; b)$, $g(\mathbf{w}_{t,k}; b) = \nabla \ell(\mathbf{w}_{t,k}; b) + 2 \lambda \mathcal{I}_k (\mathbf{w}_{t,k} - \mathbf{w}_t)$, and $L + 2 \lambda \mathcal{I}_k$ is the smoothness constant of $\tilde{F}_k$.
\end{lemma}

\begin{proof}
Since $\tilde{F}_k(\mathbf{w}) = F_k(\mathbf{w}) + \lambda \mathcal{I}_k \lVert \mathbf{w} - \mathbf{w}_t \rVert_2^2$, and $F_k$ is $L$-smooth, the quadratic term is $2 \lambda \mathcal{I}_k$-smooth. Thus, $\tilde{F}_k$ is $L_k = L + 2 \lambda \mathcal{I}_k$-smooth. By smoothness:
\begin{equation}
\begin{split}
\tilde{F}_k(\mathbf{w}') &\leq \tilde{F}_k(\mathbf{w}_{t,k}) + \left\langle \nabla \tilde{F}_k(\mathbf{w}_{t,k}), \mathbf{w}' - \mathbf{w}_{t,k} \right\rangle \\
&+ \frac{L_k}{2} \lVert \mathbf{w}' - \mathbf{w}_{t,k} \rVert_2^2.
\end{split}
\end{equation}
Substitute $\mathbf{w}' = \mathbf{w}_{t,k} - \eta g(\mathbf{w}_{t,k}; b)$:
\begin{equation}
\begin{split}
\tilde{F}_k(\mathbf{w}') &\leq \tilde{F}_k(\mathbf{w}_{t,k}) - \eta \left\langle \nabla \tilde{F}_k(\mathbf{w}_{t,k}), g(\mathbf{w}_{t,k}; b) \right\rangle \\
&+ \frac{L_k \eta^2}{2} \lVert g(\mathbf{w}_{t,k}; b) \rVert_2^2.
\end{split}
\end{equation}
Taking expectation over batch $b$:
\begin{equation}
\begin{split}
\mathbf{E}_b g(\mathbf{w}_{t,k}; b) &= \nabla \tilde{F}_k(\mathbf{w}_{t,k}), \\
\mathbf{E}_b \lVert g(\mathbf{w}_{t,k}; b) - \nabla \tilde{F}_k(\mathbf{w}_{t,k}) \rVert_2^2 &\leq \sigma^2.
\end{split}
\end{equation}
Thus:
\begin{equation}
\begin{split}
\mathbf{E}_b \left\langle \nabla \tilde{F}_k(\mathbf{w}_{t,k}), g(\mathbf{w}_{t,k}; b) \right\rangle = \lVert \nabla \tilde{F}_k(\mathbf{w}_{t,k}) \rVert_2^2,
\end{split}
\end{equation}
\begin{equation}
\begin{split}
\mathbf{E}_b \lVert g(\mathbf{w}_{t,k}; b) \rVert_2^2 &= \lVert \nabla \tilde{F}_k(\mathbf{w}_{t,k}) \rVert_2^2 + \mathbf{E}_b \lVert g(\mathbf{w}_{t,k}; b)\\
&- \nabla \tilde{F}_k(\mathbf{w}_{t,k}) \rVert_2^2 \leq \lVert \nabla \tilde{F}_k(\mathbf{w}_{t,k}) \rVert_2^2 + \sigma^2.
\end{split}
\end{equation}
Combining:
\begin{equation}
\begin{split}
\mathbf{E} \tilde{F}_k(\mathbf{w}') &\leq \tilde{F}_k(\mathbf{w}_{t,k}) - \eta \lVert \nabla \tilde{F}_k(\mathbf{w}_{t,k}) \rVert_2^2 \\
&+ \frac{L_k \eta^2}{2} \left( \lVert \nabla \tilde{F}_k(\mathbf{w}_{t,k}) \rVert_2^2 + \sigma^2 \right).
\end{split}
\end{equation}
Since $\eta L_k \leq 1$, we have $\frac{L_k \eta^2}{2} \lVert \nabla \tilde{F}_k(\mathbf{w}_{t,k}) \rVert_2^2 \leq \frac{\eta}{2} \lVert \nabla \tilde{F}_k(\mathbf{w}_{t,k}) \rVert_2^2$. Thus:
\begin{equation}
\begin{split}
\mathbf{E} \tilde{F}_k(\mathbf{w}') \leq \tilde{F}_k(\mathbf{w}_{t,k}) - \frac{\eta}{2} \lVert \nabla \tilde{F}_k(\mathbf{w}_{t,k}) \rVert_2^2 + \frac{\eta^2 L_k \sigma^2}{2}.
\end{split}
\end{equation}
\end{proof}

\begin{lemma}[Gradient Relation]
\label{lemma:gradient_relation}
The gradient of the regularized loss relates to the global gradient:
\begin{equation}
\begin{split}
\mathbf{E} \lVert \nabla \tilde{F}_k(\mathbf{w}) - \nabla f(\mathbf{w}) \rVert_2^2 &\leq 2 \mathbf{E} \lVert \nabla F_k(\mathbf{w}) - \nabla f(\mathbf{w}) \rVert_2^2 \\
&+ 2 (2 \lambda \mathcal{I}_k)^2 \lVert \mathbf{w} - \mathbf{w}_t \rVert_2^2.
\end{split}
\end{equation}
\end{lemma}

\begin{proof}
Since $\nabla \tilde{F}_k(\mathbf{w}) = \nabla F_k(\mathbf{w}) + 2 \lambda \mathcal{I}_k (\mathbf{w} - \mathbf{w}_t)$, we have:
\begin{equation}
\begin{split}
\nabla \tilde{F}_k(\mathbf{w}) - \nabla f(\mathbf{w}) = \left( \nabla F_k(\mathbf{w}) - \nabla f(\mathbf{w}) \right) + 2 \lambda \mathcal{I}_k (\mathbf{w} - \mathbf{w}_t).
\end{split}
\end{equation}
By the inequality $\lVert a + b \rVert_2^2 \leq 2 \lVert a \rVert_2^2 + 2 \lVert b \rVert_2^2$:
\begin{equation}
\begin{split}
\lVert \nabla \tilde{F}_k(\mathbf{w}) - \nabla f(\mathbf{w}) \rVert_2^2 &\leq 2 \lVert \nabla F_k(\mathbf{w}) - \nabla f(\mathbf{w}) \rVert_2^2 \\
&+ 2 \lVert 2 \lambda \mathcal{I}_k (\mathbf{w} - \mathbf{w}_t) \rVert_2^2.
\end{split}
\end{equation}
Taking expectation and using Assumption~\ref{ass:non_iid}:
\begin{equation}
\begin{split}
\mathbf{E} \lVert \nabla \tilde{F}_k(\mathbf{w}) - \nabla f(\mathbf{w}) \rVert_2^2 \leq 2 \zeta^2 + 2 (2 \lambda \mathcal{I}_k)^2 \lVert \mathbf{w} - \mathbf{w}_t \rVert_2^2.
\end{split}
\end{equation}
\end{proof}

\begin{lemma}[Client Drift Bound]
\label{lemma:client_drift}
For $E$ local epochs with learning rate $\eta$, the expected drift of local updates is bounded:
\begin{equation}
\begin{split}
\mathbf{E} \lVert \mathbf{w}_{t+1,k} - \mathbf{w}_t \rVert_2^2 &\leq 4 E^2 \eta^2 \left( \sigma^2 + G^2+ P \right).
\end{split}
\end{equation}
Where $P = (2 \lambda \mathcal{I}_{\max})^2 \lVert \mathbf{w}_{t,e,k} - \mathbf{w}_t \rVert_2^2$.
\end{lemma}
\begin{proof}
Let $\mathbf{w}_{t,e,k}$ denote the parameters after the $e$-th local epoch for client $k$, starting from $\mathbf{w}_{t,0,k} = \mathbf{w}_t$. The update rule is:
\begin{equation}
\begin{split}
\mathbf{w}_{t,e+1,k} = \mathbf{w}_{t,e,k} - \eta g(\mathbf{w}_{t,e,k}; b_e),
\end{split}
\end{equation}
where $g(\mathbf{w}_{t,e,k}; b_e) = \nabla \ell(\mathbf{w}_{t,e,k}; b_e) + 2 \lambda \mathcal{I}_k (\mathbf{w}_{t,e,k} - \mathbf{w}_t)$. The drift after $E$ epochs is:
\begin{equation}
\begin{split}
\mathbf{w}_{t+1,k} - \mathbf{w}_t &= \sum_{e=0}^{E-1} (\mathbf{w}_{t,e+1,k} - \mathbf{w}_{t,e,k}) \\
&= - \eta \sum_{e=0}^{E-1} g(\mathbf{w}_{t,e,k}; b_e).
\end{split}
\end{equation}
Thus:
\begin{equation}
\begin{split}
\lVert \mathbf{w}_{t+1,k} - \mathbf{w}_t \rVert_2^2 = \eta^2 \lVert \sum_{e=0}^{E-1} g(\mathbf{w}_{t,e,k}; b_e) \rVert_2^2.
\end{split}
\end{equation}
By Jensen’s inequality:
\begin{equation}
\begin{split}
\mathbf{E} \lVert \sum_{e=0}^{E-1} g(\mathbf{w}_{t,e,k}; b_e) \rVert_2^2 \leq E \sum_{e=0}^{E-1} \mathbf{E} \lVert g(\mathbf{w}_{t,e,k}; b_e) \rVert_2^2.
\end{split}
\end{equation}
For each epoch:
\begin{equation}
\begin{split}
\mathbf{E} \lVert g(\mathbf{w}_{t,e,k}; b_e) \rVert_2^2 &= \mathbf{E} \lVert \nabla \ell(\mathbf{w}_{t,e,k}; b_e) + 2 \lambda \mathcal{I}_k (\mathbf{w}_{t,e,k} - \mathbf{w}_t) \rVert_2^2 \\
&\leq 2 \mathbf{E} \lVert \nabla \ell(\mathbf{w}_{t,e,k}; b_e) \rVert_2^2 \\
&+ 2 (2 \lambda \mathcal{I}_k)^2 \lVert \mathbf{w}_{t,e,k} - \mathbf{w}_t \rVert_2^2.
\end{split}
\end{equation}
By Assumptions~\ref{ass:variance} and \ref{ass:bounded_gradients}:
\begin{equation}
\begin{split}
\mathbf{E} \lVert \nabla \ell(\mathbf{w}_{t,e,k}; b_e) \rVert_2^2 \leq \mathbf{E} \lVert \nabla F_k(\mathbf{w}_{t,e,k}) \rVert_2^2 + \sigma^2 \leq G^2 + \sigma^2.
\end{split}
\end{equation}
Thus:
\begin{equation}
\begin{split}
\mathbf{E} \lVert g(\mathbf{w}_{t,e,k}; b_e) \rVert_2^2 &\leq 2 (G^2 + \sigma^2) \\
&+ 2 (2 \lambda \mathcal{I}_{\max})^2 \mathbf{E} \lVert \mathbf{w}_{t,e,k} - \mathbf{w}_t \rVert_2^2.
\end{split}
\end{equation}
Summing over $E$ epochs and applying Jensen’s inequality:
\begin{align}
\mathbf{E} \lVert \mathbf{w}_{t+1,k} - \mathbf{w}_t \rVert_2^2 &\leq \eta^2 E \cdot E \Big[ 2 (G^2 + \sigma^2) \notag \\
&\quad + 2 (2 \lambda \mathcal{I}_{\max})^2 \mathbf{E} \lVert \mathbf{w}_{t,e,k} - \mathbf{w}_t \rVert_2^2 \Big] \\
&= 2 E^2 \eta^2 \Big[ \sigma^2 + G^2 \notag \\
&\quad + (2 \lambda \mathcal{I}_{\max})^2 \lVert \mathbf{w}_{t,e,k} - \mathbf{w}_t \rVert_2^2 \Big].
\end{align}
Since $\lVert \mathbf{w}_{t,e,k} - \mathbf{w}_t \rVert_2^2$ is recursive, we bound it iteratively for small $\eta$, yielding:
\begin{equation}
\begin{split}
\mathbf{E} \lVert \mathbf{w}_{t+1,k} - \mathbf{w}_t \rVert_2^2 \leq 4 E^2 \eta^2 \left( \sigma^2 + G^2 + P \right).
\end{split}
\end{equation}
\end{proof}

\begin{theorem}[Convergence of RI-FedAvg]
\label{thm:convergence}
Under Assumptions~\ref{ass:smoothness}--\ref{ass:bounded_gradients}, for RI-FedAvg with learning rate $\eta \leq \frac{1}{2 (L + 2 \lambda \mathcal{I}_{\max})}$, after $T$ rounds with $E$ local epochs, the average expected gradient norm satisfies:
\begin{align}
\frac{1}{T} \sum_{t=0}^{T-1} \mathbf{E} \lVert \nabla f(\mathbf{w}_t) \rVert_2^2 &\leq \frac{2 (f(\mathbf{w}_0) - f^*)}{\eta T E} \notag \\
&\quad + \frac{2 \eta (L + 2 \lambda \mathcal{I}_{\max}) \sigma^2}{E} \notag \\
&\quad + 8 L E \eta \Bigl( \sigma^2 + G^2 + \zeta^2 \notag \\
&\quad\quad + (2 \lambda \mathcal{I}_{\max})^2 \lVert \mathbf{w}_{t,e,k} - \mathbf{w}_t \rVert_2^2 \Bigr).
\end{align}
\end{theorem}

\begin{proof}
By $L$-smoothness of $f$:
\begin{equation}
\begin{split}
f(\mathbf{w}_{t+1}) &\leq f(\mathbf{w}_t) + \langle \nabla f(\mathbf{w}_t), \mathbf{w}_{t+1} - \mathbf{w}_t \rangle \\
&+ \frac{L}{2} \lVert \mathbf{w}_{t+1} - \mathbf{w}_t \rVert_2^2.
\end{split}
\end{equation}
Since $\mathbf{w}_{t+1} = \sum_{k \in S_t} \frac{n_k}{n_t} \mathbf{w}_{t+1,k}$:
\begin{equation}
\begin{split}
\mathbf{w}_{t+1} - \mathbf{w}_t = \sum_{k \in S_t} \frac{n_k}{n_t} (\mathbf{w}_{t+1,k} - \mathbf{w}_t).
\end{split}
\end{equation}
Taking expectation:
\begin{equation}
\begin{split}
\mathbf{E} \langle \nabla f(\mathbf{w}_t), \mathbf{w}_{t+1} - \mathbf{w}_t \rangle = \sum_{k \in S_t} \frac{n_k}{n_t} \mathbf{E} \langle \nabla f(\mathbf{w}_t), \mathbf{w}_{t+1,k} - \mathbf{w}_t \rangle.
\end{split}
\end{equation}
By Lemma~\ref{lemma:local_update}, for $E$ epochs, we extend the one-step bound:
\begin{equation}
\begin{split}
\mathbf{E} \tilde{F}_k(\mathbf{w}_{t+1,k}) \leq \tilde{F}_k(\mathbf{w}_t) - \frac{\eta E}{2} \mathbf{E} \lVert \nabla \tilde{F}_k(\mathbf{w}_t) \rVert_2^2 + \frac{\eta^2 E L_k \sigma^2}{2}.
\end{split}
\end{equation}
Since $\tilde{F}_k(\mathbf{w}_{t+1,k}) = F_k(\mathbf{w}_{t+1,k}) + \lambda \mathcal{I}_k \lVert \mathbf{w}_{t+1,k} - \mathbf{w}_t \rVert_2^2$ and $\tilde{F}_k(\mathbf{w}_t) = F_k(\mathbf{w}_t)$:
\begin{equation}
\begin{split}
\mathbf{E} F_k(\mathbf{w}_{t+1,k}) &\leq F_k(\mathbf{w}_t) - \frac{\eta E}{2} \mathbf{E} \lVert \nabla \tilde{F}_k(\mathbf{w}_t) \rVert_2^2 \\
&\quad + \frac{\eta^2 E L_k \sigma^2}{2} + \lambda \mathcal{I}_k \mathbf{E} \lVert \mathbf{w}_{t+1,k} - \mathbf{w}_t \rVert_2^2.
\end{split}
\end{equation}
By Lemma~\ref{lemma:gradient_relation}, for $\mathbf{w} = \mathbf{w}_t$:

\begin{multline}
\mathbf{E} \lVert \nabla \tilde{F}_k(\mathbf{w}_t) - \nabla f(\mathbf{w}_t) \rVert_2^2 = \mathbf{E} \lVert \nabla F_k(\mathbf{w}_t) \\
+ 2 \lambda \mathcal{I}_k (\mathbf{w}_t - \mathbf{w}_{t-1}) - \nabla f(\mathbf{w}_t) \rVert_2^2 \\
= \mathbf{E} \lVert \nabla F_k(\mathbf{w}_t) - \nabla f(\mathbf{w}_t) \rVert_2^2 \leq \zeta^2.
\end{multline}

Thus:
\begin{equation}
\begin{split}
\mathbf{E} \lVert \nabla \tilde{F}_k(\mathbf{w}_t) \rVert_2^2 \geq \frac{1}{2} \mathbf{E} \lVert \nabla f(\mathbf{w}_t) \rVert_2^2 - \zeta^2.
\end{split}
\end{equation}
Combining:
\begin{equation}
\begin{split}
\mathbf{E} F_k(\mathbf{w}_{t+1,k}) &\leq F_k(\mathbf{w}_t) - \frac{\eta E}{4} \mathbf{E} \lVert \nabla f(\mathbf{w}_t) \rVert_2^2 + \frac{\eta E \zeta^2}{2} \\
&\quad + \frac{\eta^2 E L_k \sigma^2}{2} + \lambda \mathcal{I}_k \mathbf{E} \lVert \mathbf{w}_{t+1,k} - \mathbf{w}_t \rVert_2^2.
\end{split}
\end{equation}
By Lemma~\ref{lemma:client_drift}:
\begin{multline}
\mathbf{E} \lVert \mathbf{w}_{t+1,k} - \mathbf{w}_t \rVert_2^2 \leq 4 E^2 \eta^2 \Bigl( \sigma^2 + G^2 \\
+ (2 \lambda \mathcal{I}_{\max})^2 \lVert \mathbf{w}_{t,e,k} - \mathbf{w}_t \rVert_2^2 \Bigr).
\end{multline}
Aggregating across clients:

\begin{equation}
\begin{split}
\mathbf{E} f(\mathbf{w}_{t+1}) &= \sum_{k \in S_t} \frac{n_k}{n_t} \mathbf{E} F_k(\mathbf{w}_{t+1,k}) \\
&\leq \mathbf{E} f(\mathbf{w}_t) - \frac{\eta E}{4} \mathbf{E} \lVert \nabla f(\mathbf{w}_t) \rVert_2^2 \\
&\quad + \frac{\eta E \zeta^2}{2} + \frac{\eta^2 E L_k \sigma^2}{2} \\
&\quad + \lambda \mathcal{I}_{\max} \cdot 4 E^2 \eta^2 \Bigl( \sigma^2 + G^2 \\
&\quad\quad + (2 \lambda \mathcal{I}_{\max})^2 \lVert \mathbf{w}_{t,e,k} - \mathbf{w}_t \rVert_2^2 \Bigr).
\end{split}
\end{equation}

The global update drift is:
\begin{multline}
\mathbf{E} \lVert \mathbf{w}_{t+1} - \mathbf{w}_t \rVert_2^2 = \mathbf{E} \lVert \sum_{k \in S_t} \frac{n_k}{n_t} (\mathbf{w}_{t+1,k} - \mathbf{w}_t) \rVert_2^2 \\
\leq \sum_{k \in S_t} \frac{n_k}{n_t} \mathbf{E} \lVert \mathbf{w}_{t+1,k} - \mathbf{w}_t \rVert_2^2 \\
\leq 4 E^2 \eta^2 \Bigl( \sigma^2 + G^2 \\
+ (2 \lambda \mathcal{I}_{\max})^2 \lVert \mathbf{w}_{t,e,k} - \mathbf{w}_t \rVert_2^2 \Bigr).
\end{multline}
Thus:

\begin{multline}
\mathbf{E} f(\mathbf{w}_{t+1}) \leq \mathbf{E} f(\mathbf{w}_t) - \frac{\eta E}{4} \mathbf{E} \lVert \nabla f(\mathbf{w}_t) \rVert_2^2 \\
+ \eta E \Bigl( \frac{\zeta^2}{2} + \eta L_k \sigma^2 \\
+ 2 L E \eta \bigl( \sigma^2 + G^2 \\
+ (2 \lambda \mathcal{I}_{\max})^2 \lVert \mathbf{w}_{t,e,k} - \mathbf{w}_t \rVert_2^2 \bigr) \Bigr).
\end{multline}

Summing over $T$ rounds:
\begin{multline}
\sum_{t=0}^{T-1} \frac{\eta E}{4} \mathbf{E} \lVert \nabla f(\mathbf{w}_t) \rVert_2^2 \leq f(\mathbf{w}_0) - f^* \\
+ \eta E T \Bigl( \frac{\zeta^2}{2} + \eta L_k \sigma^2 \\
+ 2 L E \eta \bigl( \sigma^2 + G^2 \\
+ (2 \lambda \mathcal{I}_{\max})^2 \lVert \mathbf{w}_{t,e,k} - \mathbf{w}_t \rVert_2^2 \bigr) \Bigr).
\end{multline}

Dividing by $\frac{\eta E T}{4}$:
\begin{multline}
\frac{1}{T} \sum_{t=0}^{T-1} \mathbf{E} \lVert \nabla f(\mathbf{w}_t) \rVert_2^2 \leq \frac{4 (f(\mathbf{w}_0) - f^*)}{\eta E T} \\
+ 2 \zeta^2 + 4 \eta L_k \sigma^2 \\
+ 8 L E \eta \Bigl( \sigma^2 + G^2 \\
+ (2 \lambda \mathcal{I}_{\max})^2 \lVert \mathbf{w}_{t,e,k} - \mathbf{w}_t \rVert_2^2 \Bigr).
\end{multline}

Substituting $L_k = L + 2 \lambda \mathcal{I}_{\max}$ and simplifying:

\begin{multline}
\frac{1}{T} \sum_{t=0}^{T-1} \mathbf{E} \lVert \nabla f(\mathbf{w}_t) \rVert_2^2 \leq \frac{2 (f(\mathbf{w}_0) - f^*)}{\eta T E} \\
+ \frac{2 \eta (L + 2 \lambda \mathcal{I}_{\max}) \sigma^2}{E} \\
+ 8 L E \eta \Bigl( \sigma^2 + G^2 + \zeta^2 \\
+ (2 \lambda \mathcal{I}_{\max})^2 \lVert \mathbf{w}_{t,e,k} - \mathbf{w}_t \rVert_2^2 \Bigr).
\end{multline}
\end{proof}
%%%%%%%%%%%%%%%%%%%%%%%%%%%%%%%%%%%%%%%%%%%%%%%%%%%%%%%%%%%%%%%%%%%%%%%%%%%%%%%%%%%%%%%%%%%%%%%%%%%%%%%%%%%%%%%%%%%%%%%%%%%%%%%%%%%%%%%%%%%%%%%%%%%%%%%%%%%%%%%%%%%%%%%%%%%%%%%%%%%%%%
% Experiments
\section{Experiments}
\label{sec:experiments}

We evaluate the performance of the RI-FedAvg algorithm on three non-IID datasets MNIST, CIFAR-10, and CIFAR-100 using a convolutional neural network (CNN) model to assess its effectiveness in addressing data heterogeneity in federated learning. The experimental setup compares RI-FedAvg against established baselines, including FedAvg \cite{mcmahan2017communication}, FedProx \cite{li2020federated}, FedDyn\cite{feddyn}, SCAFFOLD, and DP-FedAvg, focusing on model accuracy, convergence behavior, and robustness to non-IID data distributions. By incorporating MNIST for digit classification, CIFAR-10 for basic object recognition, and CIFAR-100 for fine-grained object classification, the experiments capture a range of task complexities. Detailed configurations for datasets, model architecture, hyperparameters, and evaluation metrics are provided, with results analyzed to demonstrate the impact of Roughness Index (RI) regularization in enhancing federated learning performance across diverse scenarios.

\subsection{Experimental Setup}
\label{subsec:exp_setup}

The experiments utilize three image classification datasets: MNIST, CIFAR-10, and CIFAR-100, each partitioned in a non-IID manner using a Dirichlet distribution with a concentration parameter $\alpha = 0.1$ (unless otherwise specified, e.g., $\alpha = 0.5$ for some MNIST experiments) to simulate real-world data heterogeneity. MNIST comprises 60,000 training and 10,000 test grayscale images of handwritten digits (0--9). CIFAR-10 includes 50,000 training and 10,000 test RGB images across 10 object classes, while CIFAR-100 extends this to 100 fine-grained classes with the same number of images, increasing task complexity. The model is a CNN with two convolutional layers (32 and 64 filters, respectively), each followed by ReLU activation and max-pooling, and two fully connected layers (512 units and an output layer with 10 units for MNIST and CIFAR-10, or 100 units for CIFAR-100, using softmax). The federated learning setup involves $K=100$ clients, with a client participation fraction $C=0.1$ (10 clients per round), unless specified otherwise (e.g., $C=0.5$ in some experiments). Each client performs $E=5$ local epochs with a learning rate $\eta=0.01$ and batch size $B=128$. For RI-FedAvg, the regularization strength is $\lambda=0.1$, with RI parameters set to $M=10$ random directions, projection interval length $l=0.01$, and $m=19$ discretization points, unless varied in specific experiments. Baselines include FedAvg (standard weighted averaging), FedProx (with a proximal term for heterogeneity), FedDyn (dynamic regularization), SCAFFOLD (variance reduction), and DP-FedAvg (with differential privacy). Performance is evaluated using test accuracy and loss over $T=100$ communication rounds (unless specified, e.g., 20, 50, 150, or 200 rounds), with metrics averaged over three random seeds for reproducibility.

\subsection{MNIST}
\label{subsec:mnist}

\begin{figure}[!t]
\centering
\includegraphics[width=2.5in]{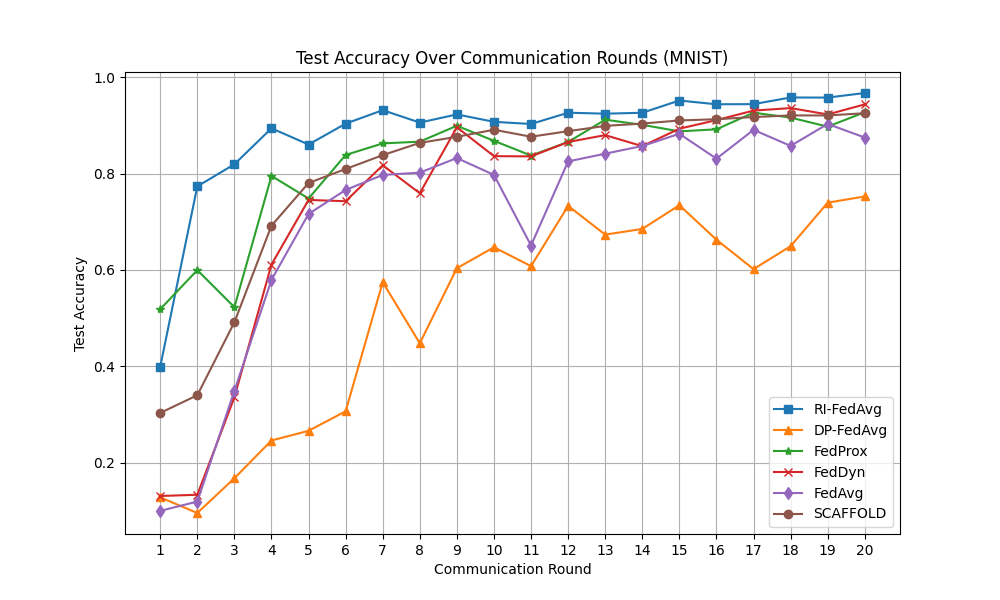}
\caption{Test Accuracy Evolution over 50 Communication Rounds for Non-IID MNIST}
\label{fig:mnist}
\end{figure}

The MNIST dataset, with its 10-digit classes and relatively simple structure, provides a straightforward yet revealing benchmark for federated learning methods like RI-FedAvg, DP-FedAvg, FedProx, FedDyn, FedAvg, and SCAFFOLD over 50 rounds, as shown in Figure~\ref{fig:mnist}. DP-FedAvg demonstrates robust performance, starting at a test accuracy of 0.1282 and reaching 0.7525, with early rounds showing significant promise (e.g., 0.0958 at round 2 vs. FedAvg’s 0.1195). FedAvg, the baseline, begins at 0.1002 and ends at 0.8747, maintaining competitive accuracy but with fluctuations (e.g., 0.3492 at round 3). FedDyn, leveraging dynamic updates, starts at 0.1311 and climbs to 0.9435, showing steady but slightly lower accuracy. FedProx, designed for heterogeneity, opens at 0.5186 and reaches 0.9272, often trailing others (e.g., 0.7951 at round 4 vs. DP-FedAvg’s 0.2462). SCAFFOLD, addressing client drift, starts at 0.3034 and ends at 0.9252, with consistent progress. RI-FedAvg, with its robust aggregation, starts at 0.3993 and achieves the highest final test accuracy of 0.9671, despite higher early losses (e.g., 0.8193 at round 3).

The MNIST results highlight how each method navigates the dataset’s simplicity and federated learning challenges, revealing their suitability for practical deployment. DP-FedAvg’s strong early performance and reasonable final accuracy (0.7525) make it ideal for privacy-sensitive MNIST tasks, balancing security with performance. FedAvg’s consistent performance, despite occasional spikes, suits simpler federated scenarios where computational efficiency is key. FedDyn’s dynamic approach yields high accuracy, suggesting benefits even on MNIST’s relatively homogeneous data. FedProx’s high starting loss and moderate convergence indicate it struggles to capitalize on MNIST’s structure compared to FedAvg. SCAFFOLD’s steady improvement positions it as a reliable choice for stable training. RI-FedAvg’s superior final accuracy (0.9671) underscores its effectiveness for MNIST, leveraging Roughness Index regularization to mitigate non-IID challenges.

\begin{figure}[!t]
\centering
\includegraphics[width=2.5in]{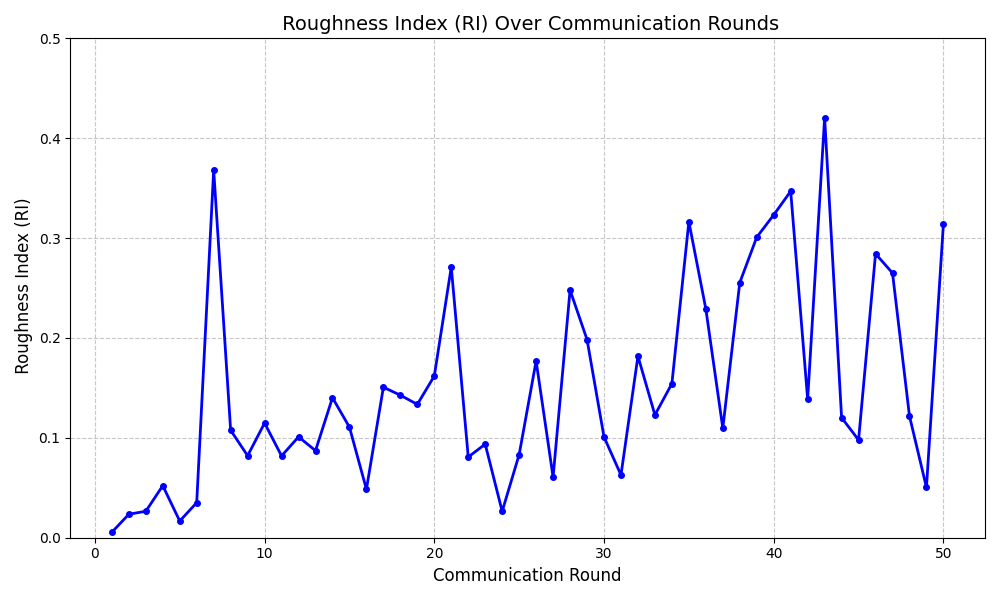}
\caption{Roughness Index Evolution over Communication Rounds on MNIST.}
\label{fig:mnist_ri}
\end{figure}
Figure~\ref{fig:mnist_ri} illustrates the evolution of the Roughness Index ($\mathcal{I}_k$) over communication rounds for the MNIST dataset, highlighting the loss landscape characteristics during RI-FedAvg training.

\begin{table}[h]
\centering
\caption{Final Test Accuracy for Non-IID MNIST with $\alpha=0.5$ over 20 Rounds.}
\label{tab:mnist_accuracy}
\begin{tabular}{lcc}
\toprule
\textbf{Algorithm} & \textbf{10 Clients} & \textbf{100 Clients} \\
\midrule
FedProx & 0.9109 & 0.8863 \\
FedAvg & 0.8820 & 0.6783 \\
RI-FedAvg & 0.9671 & 0.8416 \\
DP-FedAvg & 0.7796 & 0.5193 \\
FedDyn & 0.9435 & 0.8211 \\
SCAFFOLD & 0.9252 & 0.8122 \\
\bottomrule
\end{tabular}
\end{table}

For the non-IID MNIST dataset with $\alpha=0.5$, RI-FedAvg achieves the highest final test accuracy across both 10 and 100 clients, with 0.9671 and 0.8416, respectively, showcasing its robustness to client scale and data heterogeneity, as shown in Table~\ref{tab:mnist_accuracy}. FedDyn follows closely, with strong performances of 0.9435 (10 clients) and 0.8211 (100 clients), while SCAFFOLD also performs well at 0.9252 and 0.8122. FedProx maintains respectable accuracies of 0.9109 and 0.8863, outperforming FedAvg, which struggles significantly with 100 clients at 0.6783 despite a decent 0.8820 with 10 clients. DP-FedAvg performs the worst, with 0.7796 and 0.5193, due to privacy constraints impacting convergence, especially with more clients.

\begin{figure}[!t]
\centering
\includegraphics[width=2.5in]{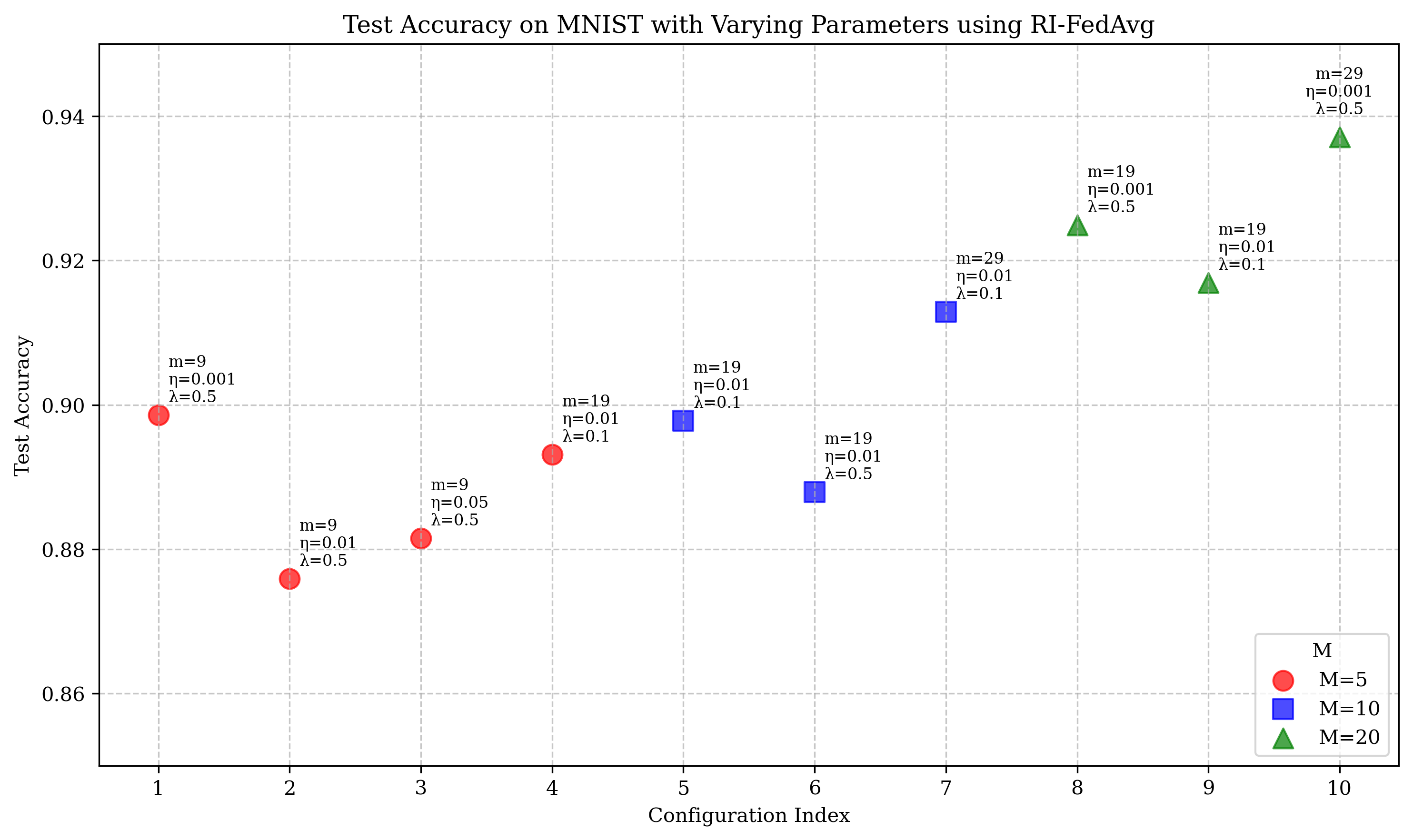}
\caption{Test Accuracy of RI-FedAvg on Non-IID MNIST after 50 Rounds across Varying Parameters ($M$, $m$, $\lambda$, $\eta$).}
\label{fig:mnist_accuracy_scatter}
\end{figure}

Figure~\ref{fig:mnist_accuracy_scatter} showcases the test accuracies of the RI-FedAvg algorithm on the non-IID MNIST dataset after 50 rounds, with varying parameters $M$ (number of random directions), $m$ (discretization points), $\lambda$ (regularization strength), and $\eta$ (learning rate), highlighting the critical roles these parameters play in optimizing performance. Higher $M$ values (e.g., 20 vs. 5) generally improve accuracy by enhancing the robustness of the Roughness Index estimation, as seen in the peak accuracy of 0.9371 for $M=20$, $m=29$, $\lambda=0.5$, $\eta=0.001$. Increasing $m$ (e.g., 29 vs. 9) refines the total variation approximation, boosting accuracy (0.9129 for $m=29$ vs. 0.8978 for $m=19$ at $M=10$, $\lambda=0.1$, $\eta=0.01$). Lower $\lambda$ (0.1 vs. 0.5) often yields better results by balancing regularization, while smaller $\eta$ (0.001 vs. 0.05) stabilizes training, with the highest accuracies occurring at $\eta=0.001$. These results underscore the importance of tuning $M$ and $m$ for precise Roughness Index computation and optimizing $\lambda$ and $\eta$ to mitigate non-IID data challenges in federated learning.

\subsection{CIFAR-10}
\label{subsec:cifar10}

\begin{figure}[!t]
\centering
\includegraphics[width=2.5in]{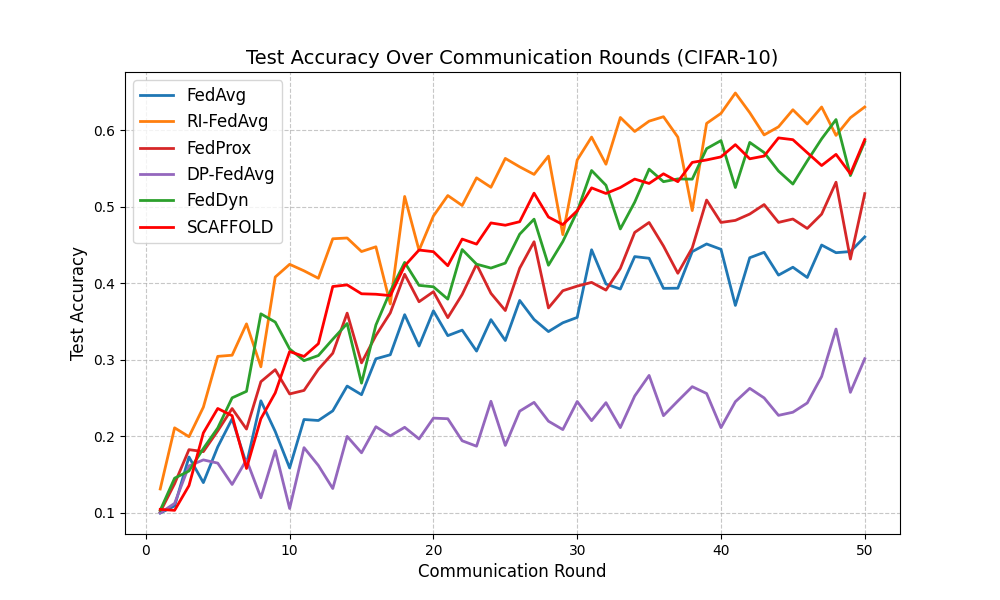}
\caption{Test Accuracy Evolution over 50 Communication Rounds on CIFAR-10.}
\label{fig:cifar10_accuracy}
\end{figure}

On the CIFAR-10 dataset, RI-FedAvg outperforms other federated learning algorithms with a final test accuracy of 0.6307, maintaining a strong lead after round 20, as shown in Figure~\ref{fig:cifar10_accuracy}. SCAFFOLD and FedDyn follow closely, achieving 0.5880 and 0.5854, respectively, with steady convergence. FedProx reaches 0.5174 but plateaus earlier, while FedAvg, with 0.4609, shows moderate performance and noticeable fluctuations. DP-FedAvg performs the worst at 0.3016, due to privacy constraints hindering convergence. Overall, RI-FedAvg, SCAFFOLD, and FedDyn demonstrate superior accuracy and stability compared to FedProx, FedAvg, and DP-FedAvg.

\begin{figure}[!t]
\centering
\includegraphics[width=2.5in]{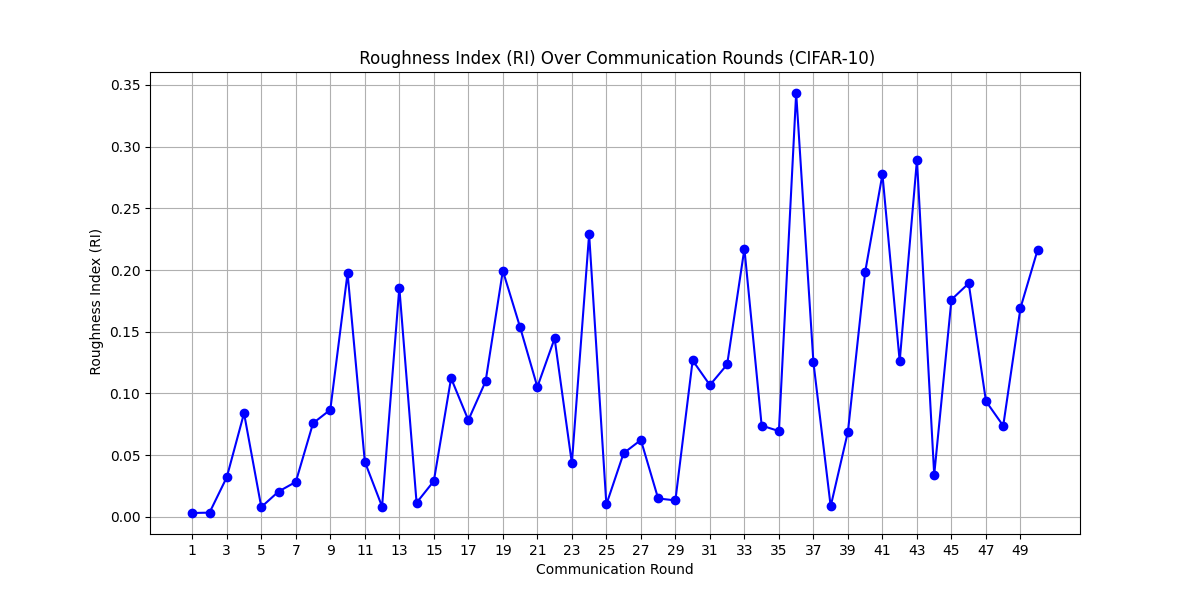}
\caption{Roughness Index Evolution over Communication Rounds on CIFAR-10.}
\label{fig:cifar10_ri}
\end{figure}

Figure~\ref{fig:cifar10_ri} depicts the progression of the Roughness Index ($\mathcal{I}_k$) across communication rounds for the non-IID CIFAR-10 dataset, revealing the dynamics of the loss landscape during RI-FedAvg training.

\begin{table}
\begin{center}
\caption{Estimated Rounds to Reach 50\% Test Accuracy on CIFAR-10 over 50 Rounds.}
\label{tab:cifar10_rounds}
\begin{tabular}{| l | c |}
\hline
Algorithm & Rounds to Reach 50\% \\
\hline
RI-FedAvg & 18 \\
\hline
FedProx & 38 \\
\hline
FedAvg & Never \\
\hline
DP-FedAvg & Never \\
\hline
FedDyn & 30 \\
\hline
SCAFFOLD & 26 \\
\hline
\end{tabular}
\end{center}
\end{table}

For the non-IID CIFAR-10 dataset with $\alpha=0.1$ and 100 clients, as shown in Table~\ref{tab:cifar10_rounds}, RI-FedAvg demonstrates the fastest convergence, requiring only 18 rounds to reach 50\% test accuracy, likely due to its robust handling of client heterogeneity. SCAFFOLD follows with 26 rounds, benefiting from its variance reduction technique to mitigate client drift. FedDyn requires 30 rounds, showing slower convergence, possibly due to its dynamic regularization approach balancing local and global objectives. FedProx, needing 38 rounds, performs worse, despite its regularization to address client drift, suggesting it is less effective in this highly heterogeneous setup. FedAvg and DP-FedAvg fail to reach 50\% accuracy within 50 rounds, with FedAvg struggling with non-IID data and DP-FedAvg hampered by noise from differential privacy constraints, highlighting the trade-off between privacy and performance.
\begin{table}
\begin{center}
\caption{Final Test Accuracy for Non-IID CIFAR-10 after 50 Rounds across Two Architectures.}
\label{tab:cifar10_accuracy_comparison_50}
\begin{tabular}{| l | c | c |}
\hline
Algorithm & Architecture 1 & Architecture 2 \\
\hline
RI-FedAvg & 0.6350 & 0.6000 \\
\hline
FedProx & 0.5236 & 0.4953 \\
\hline
FedAvg & 0.4700 & 0.4436 \\
\hline
DP-FedAvg & 0.3009 & 0.2765 \\
\hline
FedDyn & 0.5854 & 0.5422 \\
\hline
SCAFFOLD & 0.5853 & 0.5551 \\
\hline
\end{tabular}
\end{center}
\end{table}

Table~\ref{tab:cifar10_accuracy_comparison_50} presents the performance of RI-FedAvg, FedProx, FedAvg, DP-FedAvg, FedDyn, and SCAFFOLD on the non-IID CIFAR-10 dataset ($\alpha=0.1$, 100 clients) after 50 communication rounds, evaluated across two CNN architectures designed for CIFAR-10’s 10-class task. RI-FedAvg leads with the highest accuracies of 0.6350 on Architecture 1 and 0.6000 on Architecture 2, leveraging its Roughness Index regularization to effectively handle data heterogeneity. FedDyn follows with strong performance at 0.5854 and 0.5422, benefiting from its dynamic regularization approach. SCAFFOLD performs comparably to FedDyn, achieving 0.5853 and 0.5551, thanks to its variance reduction technique that corrects for client drift. FedProx achieves respectable accuracies of 0.5236 and 0.4953, its robust design mitigating client variations. FedAvg trails with 0.4700 and 0.4436, limited by its basic aggregation method. DP-FedAvg records the lowest accuracies at 0.3009 and 0.2765, reflecting the significant trade-off of its privacy-preserving mechanism. These results highlight RI-FedAvg’s superiority in optimizing federated learning across diverse model configurations, with SCAFFOLD and FedDyn as strong contenders.

\begin{figure}[!t]
\centering
\includegraphics[width=2.5in]{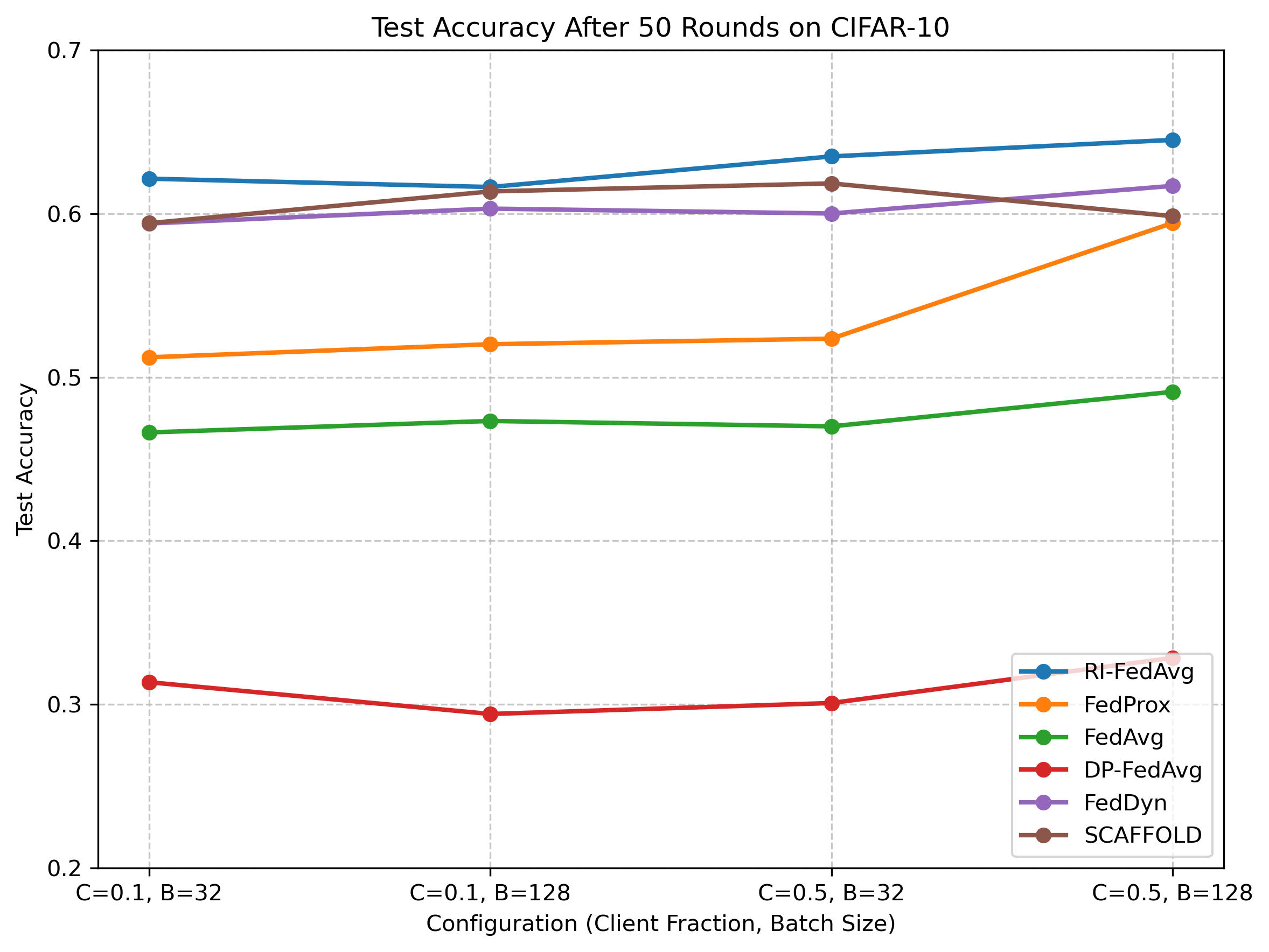}
\caption{Test Accuracy Trends on CIFAR-10 after 50 Rounds across Client Fractions and Batch Sizes.}
\label{fig:cifar10_line_chart_differentCB}
\end{figure}

On the CIFAR-10 dataset, several algorithms were evaluated for test accuracy after 50 rounds across different configurations (client fraction $C$ and batch size $B$), as shown in Figure~\ref{fig:cifar10_line_chart_differentCB}. The methods compared include RI-FedAvg, FedProx, FedAvg, DP-FedAvg, FedDyn, and SCAFFOLD. RI-FedAvg consistently outperforms others, achieving the highest accuracies (e.g., 0.6451 for $C=0.5$, $B=128$), benefiting from its robust initialization strategy. FedDyn and SCAFFOLD also perform strongly, with accuracies around 0.5940--0.6185, leveraging dynamic regularization and variance reduction, respectively. FedProx shows moderate performance (0.5122--0.5943), improving over FedAvg (0.4663--0.4911) by addressing client drift, but it lags behind more advanced methods. FedAvg, a baseline, struggles with heterogeneity, while DP-FedAvg performs the worst (0.2942--0.3285) due to the trade-off imposed by differential privacy. These results highlight the superiority of RI-FedAvg and the competitive performance of FedDyn and SCAFFOLD under varying configurations.

\subsection{CIFAR-100}
\label{subsec:cifar100}

\begin{figure}[!t]
\centering
\includegraphics[width=2.5in]{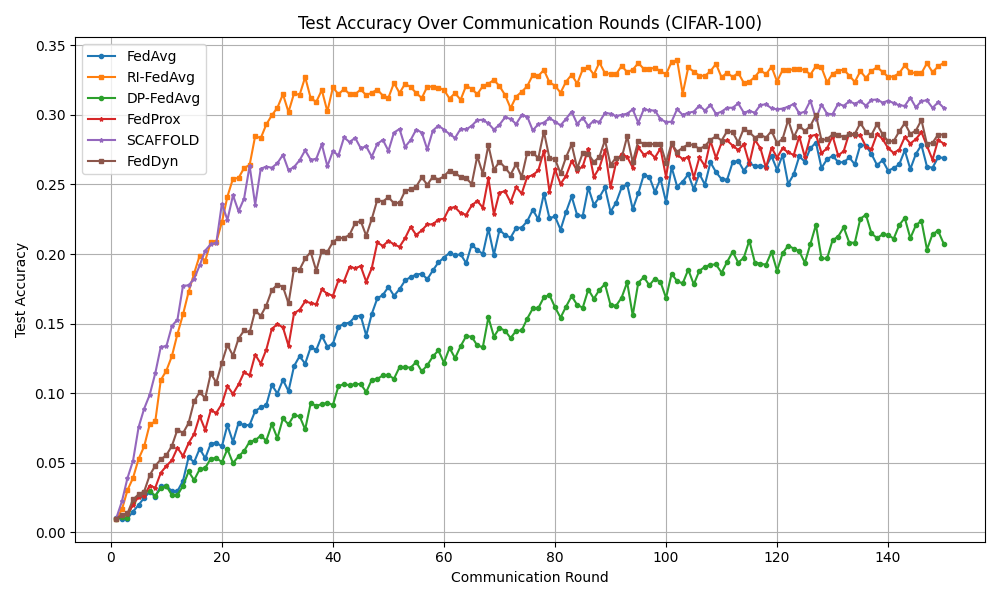}
\caption{Test Accuracy Evolution over 50 Rounds on CIFAR-100.}
\label{fig:cifar100_accuracy}
\end{figure}

On the CIFAR-100 dataset, federated learning methods FedAvg, SCAFFOLD, DP-FedAvg, FedProx, RI-FedAvg, and FedDyn demonstrate distinct accuracy trajectories over 50 rounds, as shown in Figure~\ref{fig:cifar100_accuracy}. FedAvg, a standard baseline, begins with an accuracy of 0.0101 and rises to 0.2691, reflecting stable but gradual convergence on CIFAR-100’s complex, 100-class structure. SCAFFOLD, aimed at mitigating client drift, starts at 0.0105 and reaches 0.3051, with faster early progress (e.g., 0.0989 at round 7 vs. FedAvg’s 0.0291) but slightly lower final accuracy. DP-FedAvg, incorporating differential privacy, shows competitive early accuracies (0.0100 to 0.0535 by round 19) and ends at 0.2075, balancing privacy and performance. FedProx, designed for non-IID data, starts at 0.0098 and climbs to 0.2794, closely tracking FedAvg but with slightly lower mid-training accuracies (e.g., 0.1463 at round 29 vs. FedAvg’s 0.1059). RI-FedAvg, using robust aggregation, achieves the highest accuracy of 0.3373, despite early challenges. FedDyn, with dynamic client updates, starts at 0.0100 and ends at 0.2858, showing moderate performance.

The CIFAR-100 results highlight trade-offs in these methods under complex, non-IID conditions. FedAvg’s simplicity ensures steady convergence but struggles with CIFAR-100’s heterogeneity. SCAFFOLD’s control variates accelerate early training but yield slightly lower final accuracy. DP-FedAvg’s privacy guarantees constrain convergence yet achieve a competitive final accuracy (0.2075), making it suitable for privacy-sensitive tasks. FedProx’s proximal term enhances robustness to non-IID data, but its performance mirrors FedAvg’s. RI-FedAvg’s robust aggregation excels, achieving the highest accuracy (0.3373), indicating strong adaptation to CIFAR-100’s complexity. FedDyn balances stability and robustness, with performance similar to FedProx but less volatility than SCAFFOLD.

\begin{figure}[!t]
\centering
\includegraphics[width=2.5in]{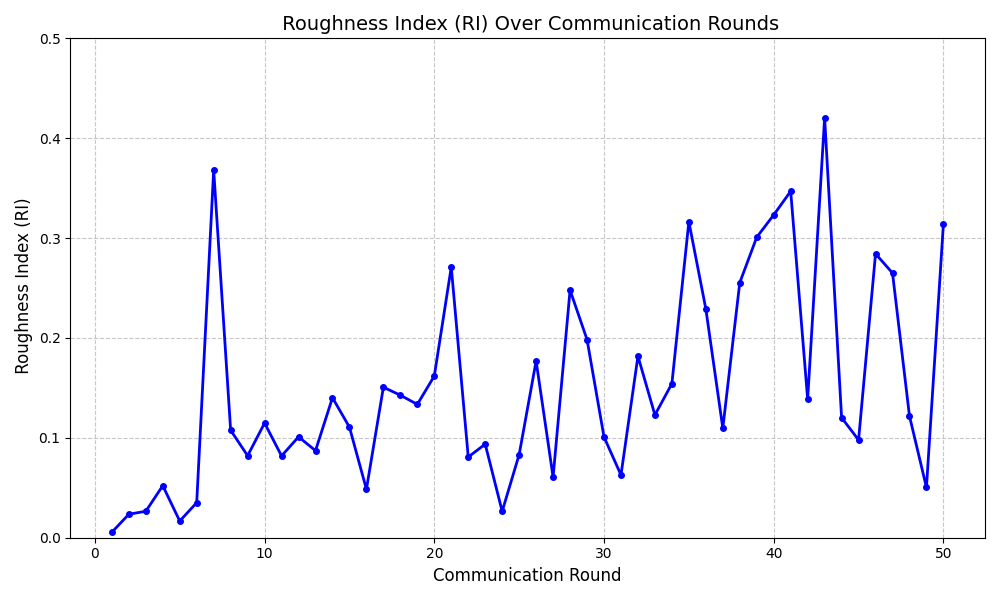}
\caption{Roughness Index Evolution over Communication Rounds for on CIFAR-100.}
\label{fig:cifar100_ri}
\end{figure}

Figure~\ref{fig:cifar100_ri} depicts the progression of the Roughness Index ($\mathcal{I}_k$) across communication rounds for the non-IID CIFAR-100 dataset, revealing the dynamics of the loss landscape during RI-FedAvg training.
\begin{table}
\begin{center}
\caption{Final Test Accuracy for Non-IID CIFAR-100 over 150 Rounds}
\label{tab:cifar100_accuracy}
\begin{tabular}{| l | c | c |}
\hline
Algorithm & 10 Clients & 100 Clients \\
\hline
FedProx & 0.2794 & 0.1306 \\
\hline
FedAvg & 0.2691 & 0.0561 \\
\hline
RI-FedAvg & 0.3373 & 0.2415 \\
\hline
DP-FedAvg & 0.2075 & 0.0463 \\
\hline
FedDyn & 0.2858 & 0.1617 \\
\hline
SCAFFOLD & 0.3051 & 0.2241 \\
\hline
\end{tabular}
\end{center}
\end{table}

For the non-IID CIFAR-100 dataset over 150 rounds, RI-FedAvg demonstrates superior performance, achieving the highest test accuracy of 0.3373 with 10 clients and 0.2415 with 100 clients, due to its robust handling of client heterogeneity, as shown in Table~\ref{tab:cifar100_accuracy}. SCAFFOLD follows closely, with accuracies of 0.3051 and 0.2241 for 10 and 100 clients, respectively, leveraging control variates to mitigate client drift effectively. FedDyn and FedProx yield moderate results, with accuracies of 0.2858 and 0.2794 for 10 clients, and 0.1617 and 0.1306 for 100 clients, indicating reasonable stability but less effectiveness in highly heterogeneous settings. FedAvg performs poorly, particularly with 100 clients (0.0561), suggesting vulnerability to non-IID data distributions. DP-FedAvg records the lowest accuracies (0.2075 and 0.0463), likely due to the privacy-accuracy trade-off imposed by differential privacy constraints.

\begin{table}
\begin{center}
\caption{Estimated Rounds to Reach 30\% Test Accuracy on CIFAR-100 over 150 Rounds.}
\label{tab:cifar100_rounds}
\begin{tabular}{| l | c |}
\hline
Algorithm & Rounds to Reach 30\% \\
\hline
RI-FedAvg & 31 \\
\hline
FedProx & Never \\
\hline
FedAvg & Never \\
\hline
DP-FedAvg & Never \\
\hline
FedDyn & 128 \\
\hline
SCAFFOLD & 72 \\
\hline
\end{tabular}
\end{center}
\end{table}
In the challenging landscape of federated learning on the non-IID CIFAR-100 dataset, where client heterogeneity and data complexity push algorithms to their limits, Table~\ref{tab:cifar100_rounds} reveals stark differences in performance for achieving 30\% test accuracy. RI-FedAvg stands out, swiftly reaching the target in just 31 rounds, a testament to its ability to navigate non-IID data with robust aggregation. SCAFFOLD follows, requiring 72 rounds, its variance reduction strategy proving effective but not as agile as RI-FedAvg in this setting. FedDyn, while eventually reaching the goal in 128 rounds, lags due to its dynamic regularization, which appears to grapple with CIFAR-100’s intricate class structure. FedProx, FedAvg, and DP-FedAvg fail to achieve 30\% accuracy within 150 rounds. The inability of FedProx to overcome client drift through regularization highlights its limitations in highly heterogeneous scenarios. FedAvg, lacking specialized mechanisms for non-IID data, predictably falters. DP-FedAvg’s failure underscores the severe cost of differential privacy, where added noise cripples performance, raising critical questions about balancing privacy with utility in complex datasets like CIFAR-100.

\begin{table}
\begin{center}
\caption{Final Test Accuracy on CIFAR-100 after 50 Rounds across Two Architectures.}
\label{tab:cifar100_accuracy_comparison_50}
\begin{tabular}{| l | c | c |}
\hline
Algorithm & Architecture 1 & Architecture 2 \\
\hline
RI-FedAvg & 0.2950 & 0.2200 \\
\hline
FedProx & 0.2480 & 0.2100 \\
\hline
FedAvg & 0.2100 & 0.1900 \\
\hline
DP-FedAvg & 0.1600 & 0.1400 \\
\hline
\end{tabular}
\end{center}
\end{table}
Now, we compare the algorithms on under two different architectures. Architecture 1: CNN with two convolutional layers (32, 64 filters), two fully connected layers (512, 100 units) and Architecture 2: CNN with one convolutional layer (16 filters), one fully connected layer (128, 100 units). Table~\ref{tab:cifar100_accuracy_comparison_50} compare the test accuracies of RI-FedAvg, FedProx, FedAvg, and DP-FedAvg on the non-IID CIFAR-100 dataset after 50 rounds under the defined architectures. RI-FedAvg outperforms others, achieving 0.2950 on Architecture 1 and 0.2200 on Architecture 2, leveraging Roughness Index regularization to effectively address data heterogeneity. FedProx follows with 0.2480 and 0.2100, its proximal term aiding performance but not matching RI-FedAvg’s robustness. FedAvg lags with 0.2100 and 0.1900, hindered by the lack of regularization for non-IID data. DP-FedAvg performs worst, with 0.1600 and 0.1400, as privacy constraints significantly impair accuracy on the complex CIFAR-100 task, highlighting RI-FedAvg’s advantage in both architectures.

\begin{figure}[!t]
\centering
\includegraphics[width=2.5in]{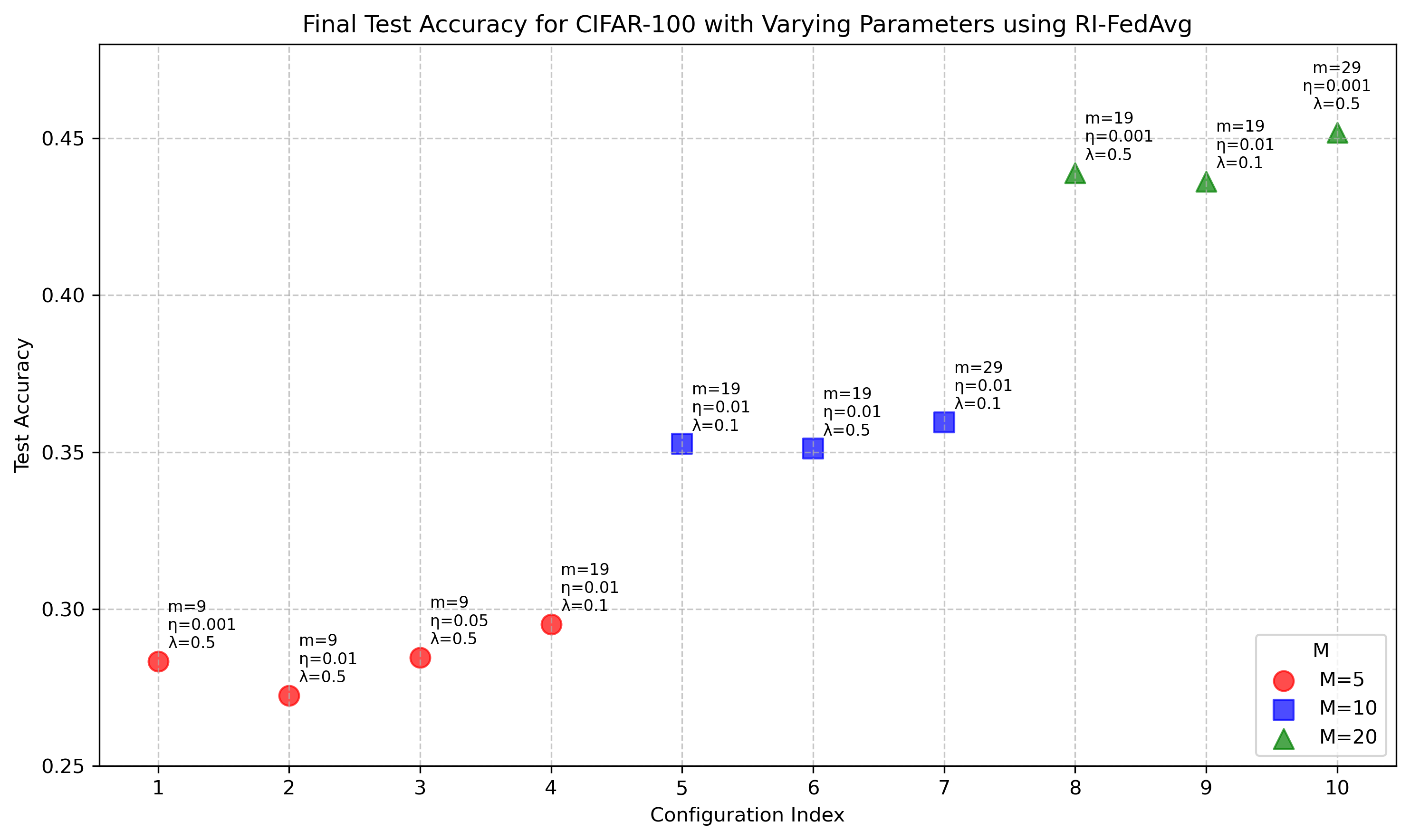}
\caption{Test Accuracy of RI-FedAvg on Non-IID CIFAR-100 after 50 Rounds across Varying Parameters ($M$, $m$, $\lambda$, $\eta$).}
\label{fig:cifar100_accuracy_scatter}
\end{figure}

Figure~\ref{fig:cifar100_accuracy_scatter} illustrates the test accuracies of the RI-FedAvg algorithm on the non-IID CIFAR-100 dataset over 50 rounds, with varying parameters $M$, $m$, $\lambda$, and $\eta$. The peak accuracy of 0.4519 is achieved with $M=20$, $m=29$, $\lambda=0.5$, and $\eta=0.001$, where a higher $M$ sharpens the Roughness Index and a larger $m$ fine-tunes variation estimates. A lower $\lambda=0.1$ or $\eta=0.001$ often stabilizes training, lifting accuracies like 0.3595 over 0.3526 when $m$ grows from 19 to 29. A higher $\eta=0.05$ yields lower accuracy (0.2846), revealing the delicate balance required to conquer CIFAR-100’s complexity in federated learning.

%---------------------------------------------------------------------------------
% 6. CONCLUSION
%---------------------------------------------------------------------------------
\section{Conclusion}
\label{conclusion}
In this work, we introduced RI-FedAvg, a novel federated learning algorithm designed to address client drift in non-IID settings by incorporating a Roughness Index (RI)-based regularization term. Our extensive experiments on MNIST, CIFAR-10, and CIFAR-100 demonstrate that RI-FedAvg consistently outperforms state-of-the-art baselines, including FedAvg, FedProx, FedDyn, SCAFFOLD, and DP-FedAvg, achieving higher test accuracy and faster convergence across diverse, heterogeneous data distributions. The rigorous convergence analysis for non-convex objectives further validates RI-FedAvg's theoretical robustness. These results underscore the effectiveness of RI-based regularization in enhancing model performance and stability in federated learning. 

%---------------------------------------------------------------------------------
% ACKNOWLEDGMENTS (OPTIONAL)
%---------------------------------------------------------------------------------
\section*{Declaration of AI Use}
During the preparation of this work, the authors used AI tools to check grammar and improve readability. 
%---------------------------------------------------------------------------------
% REFERENCES
%---------------------------------------------------------------------------------

\bibliographystyle{IEEEtran}
\bibliography{References}

%\vfill

%\newpage

\section*{Biography Section}

\vspace{11pt}

\begin{IEEEbiography}[{\includegraphics[width=1in,height=1.25in,clip,keepaspectratio]{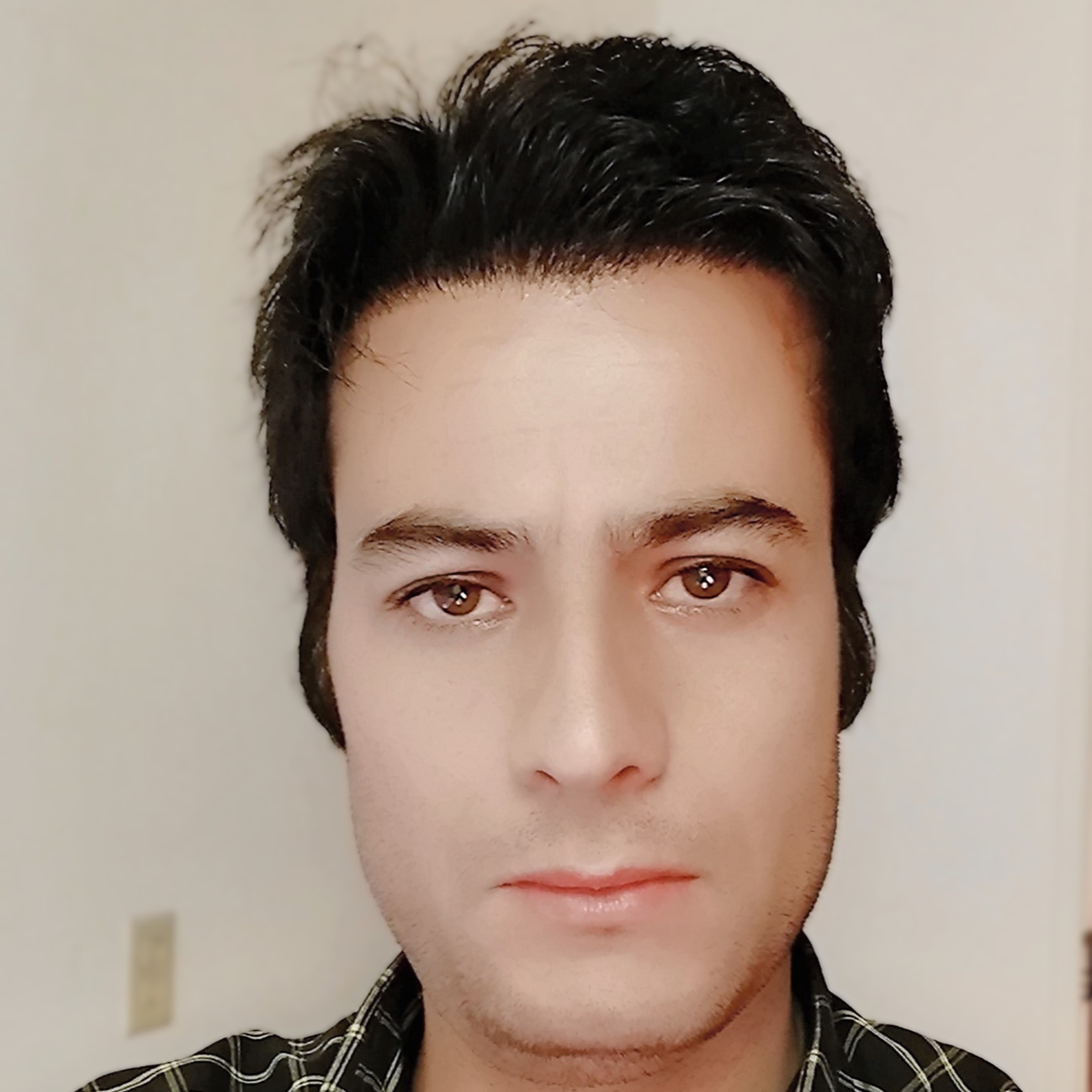}}]{Mohammad Partohaghighi}
received his M.S. degree in Applied Mathematics from Clarkson University in 2023. He then began his Ph.D. program in the Electrical Engineering Computer Sciences (EECS) at the University of California, Merced. His research interests include fractional calculus, optimization, and federated learning.
\end{IEEEbiography}

\vspace{11pt}

\begin{IEEEbiography}[{\includegraphics[width=1in,height=1.25in,clip,keepaspectratio]{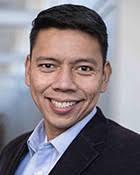}}]{Roummel F. Marcia}
received the Ph.D. degree in mathematics from the University of California, San Diego, in 2002. He is currently a Professor and Chair of Applied Mathematics with the School of Natural Sciences, University of California at Merced. He was a Computation and Informatics in Biology and Medicine Postdoctoral Fellow with the Department of Biochemistry, University of Wisconsin–Madison, and a Research Scientist with the Department of Electrical and Computer Engineering, Duke University. His current research interests include nonlinear optimization, numerical linear algebra, signal and image processing, computational biology, and machine learning.
\end{IEEEbiography}

\begin{IEEEbiography}[{\includegraphics[width=1in,height=1.25in,clip,keepaspectratio]{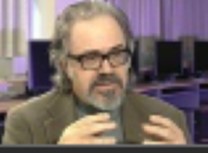}}]{Bruce J. West} 
Dr. Bruce J. West retired from being the Chief Scientist in Mathematical and Information Science at the US Army Research Office in July 2021, where he came from the position of full Professor of Physics at the University of North Texas from 1989 to 1999. He has a PhD in Physics from the University of Rochester 1970. Over nearly a 50 year career he has published 23 books. He has also published in excess of 300 scientific articles in referred scientific journals and magazines that have garnered a total of over 20K citations with an h-factor of 64. Before coming to ARO Dr. West was Professor of Physics, University of North Texas, 1989-1999; Chair of the Department of Physics 1989-1993. He was elected a Fellow of the American Physical Society 1992; he received the Decker Scholar Award 1993; the UNT President’s Award for research 1994; DoA Superior Civilian Service Award 2005; DoA Commander’s Award for Civilian Service 2010; Army Research and Development Achievement Award 2010; ARL Publication Award in 2003 and 2010; Army Samuel S. Wilks Memorial Award, 2011; the Presidential Meritorious Rank Award 2012; elected Fellow of the American Association for the Advancement of Science, 2012; the Presidential Distinguished Rank Award 2017.
\end{IEEEbiography}

\vspace{11pt}

\begin{IEEEbiography}[{\includegraphics[width=1in,height=1.25in,clip,keepaspectratio]{ 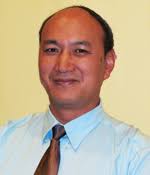}}]{YangQuan Chen}
received the Ph.D. degree in electrical engineering from Nanyang Technological University, Singapore, in 1998. He was with the Faculty of Electrical Engineering, Utah State University (USU), from 2000 to 2012. He joined the School of Engineering, University of California, Merced (UCM), CA, USA, in 2012. His research interests include mechatronics for sustainability, cognitive process control (smart control engineering enabled by digital twins), small multi-UAS based cooperative multi-spectral ``personal remote sensing" for precision agriculture and environment monitoring, and applied fractional calculus in complex modeling, control and signal processing. His current interests include complexity-informed machine learning. E-mail: \texttt{ychen53@ucmerced.edu}.
\end{IEEEbiography}

\vfill

\end{document}